\UseRawInputEncoding
\documentclass[twocolumn,preprint]{elsarticle}
\usepackage{array}
\usepackage{lineno,hyperref}
\usepackage{amsmath,amsthm,amssymb,amsfonts}
\usepackage[ruled,vlined]{algorithm2e}
\usepackage{graphicx}
\usepackage{textcomp}
\usepackage{float}
\usepackage{diagbox}
\usepackage{caption}
\usepackage{subcaption}
\usepackage{tabularx}
\usepackage{algpseudocode}
\usepackage{xcolor}
\newtheorem{thm}{Theorem}
\newtheorem{prop}{Proposition}
\newtheorem{mydef}{Definition}

\graphicspath{ {images/} }
\journal{Knowledge-based Systems}

\begin{document}

\twocolumn[{
\begin{frontmatter}

\renewcommand{\thefootnote}{\fnsymbol{footnote}}

\title{MC-GEN: Multi-level Clustering for Private Synthetic Data Generation}

\author{Mingchen Li\footnotemark[1], Di Zhuang\footnotemark[2], and J. Morris Chang\\
{\tt\small mingchenli@usf.edu, zhuangdi1990@gmail.com, chang5@usf.edu}\\}




\begin{abstract}
With the development of machine learning and data science, data sharing is very common between companies and research institutes to avoid data scarcity. However, sharing original datasets that contain private information can cause privacy leakage. A reliable solution is to utilize private synthetic datasets which preserve statistical information from original datasets. In this paper, we propose MC-GEN, a privacy-preserving synthetic data generation method under differential privacy guarantee for machine learning classification tasks. MC-GEN applies multi-level clustering and differential private generative model to improve the utility of synthetic data. In the experimental evaluation, we evaluated the effects of parameters and the effectiveness of MC-GEN. The results showed that MC-GEN can achieve significant effectiveness under certain privacy guarantees on multiple classification tasks. Moreover, we compare MC-GEN with three existing methods. The results showed that MC-GEN outperforms other methods in terms of utility.
\end{abstract}

\begin{keyword}
\texttt{synthetic data generation, differential privacy, machine learning}
\end{keyword}
\end{frontmatter}
}]

\renewcommand{\thefootnote}{\fnsymbol{footnote}}
\footnotetext[1]{This work was done by Mingchen Li and Di Zhuang when they were at the University of South Florida. }
\footnotetext[2] {Di Zhuang is currently with Snap Inc., 2772 Donald Douglas Loop N, Santa
Monica, CA 90405.}

\section{Introduction}\label{sec:introduction}
Machine learning has become an important technology in many fields, such as medical diagnosis, fraud detection, and product analysis. As a data-driven approach, data is considered as the fuel of machine learning algorithm~\cite{talavera2022data}, and the performance of the machine learning model often depends on the amount of data. To ensure the performance of machine learning, data sharing happens very often among organizations with similar research interests. For instance, if the research institutes do not have sufficient data for an illness diagnostic system, hospitals can share their data (patient records) to make up for the data gap. However, such data usually carry private information that can cause privacy leakage. Hence, sharing data in a privacy-preserving way for machine learning is of vital importance.

Sharing synthetic datasets generated from original datasets is a common way to protect data privacy. However, the individual information can be easily inferred with some background knowledge. To prevent this issue, differential privacy (DP)\cite{dwork2008differential} has been widely used as a strong and provable privacy guarantee at the individual sample level. Synthetic data releasing under differential privacy emerges as a reliable solution for privacy-preserving data sharing in machine learning to protect individual records in the original datasets. It allows the data owner to publish synthetic datasets to data users without privacy concerns, and data users can make use of synthetic datasets for different purposes, such as machine learning, data mining, etc.

Designing a powerful privacy-preserving synthetic data generation method for machine learning purposes is of great challenge. First, a privacy-preserving method usually introduces perturbations on data samples that hurt the utility of data. Mitigating the perturbation to reach a certain level of utility is not easy. Second, machine learning has multiple tasks, like support vector machine, logistic regression, random forest, and k-nearest neighbor. An effective synthetic data generation method should be applicable to different tasks. Third, some data has a complex distribution. Forming an accurate generator based on the whole data distribution is hard.

In recent years, several works related to private synthetic data release have been proposed in the literature~\cite{hardt2012simple},~\cite{dwork2011differential},~\cite{taneja2008diffgen},~\cite{su2018privpfc},~\cite{bindschaedler2017plausible},~\cite{zhang2017privbayes},~\cite{chanyaswad2017coupling},~\cite{soria2017differentially}. One kind of method is to produce the synthetic database that preserved privacy and utility regarding query sets. These algorithm~\cite{hardt2012simple},~\cite{dwork2011differential} are usually designed to answer different query classes, and data samples in the dataset are not actually released. In~\cite{taneja2008diffgen},~\cite{su2018privpfc}, they proposed algorithms for synthetic data release based on noisy histograms, which are more focused on the categorical feature variables. Algorithm~\cite{bindschaedler2017plausible},~\cite{zhang2017privbayes},~\cite{soria2017differentially} generate synthetic dataset under statistical model with some preprocessing on the original datasets. However, they only consider generating the synthetic data based on the whole data distribution.

Our primary motivation is to allow data owners to share their datasets without privacy concerns by using synthetic datasets instead of original datasets. We proposed Multiple-level Clustering based GENerator for synthetic data (MC-GEN), an approach that uses multi-level clustering (sample level and feature level) and differentially private multivariate Gaussian generative model to release synthetic datasets which satisfy $\epsilon$-differential privacy. Extensive experiments on different original datasets have been conducted to evaluate MC-GEN and its parameters. The results demonstrate that synthetic datasets generated by MC-GEN maintain the utility of classification tasks while preserving privacy. Furthermore, we compared MC-GEN with other existing methods, and our results show that MC-GEN outperforms other existing methods in terms of effectiveness.

The main contributions of MC-GEN are summarized as follows:
\begin{itemize}
  \item We proposed and released an innovative, effective synthetic data generation method, MC-GEN, which allows the data owners to share synthetic datasets for multiple classification tasks without privacy concerns\footnote[1]{https://github.com/mingchenli/MCGEN-Private-Synthetic-Data-Generator}.
  \item We demonstrated that applying feature clustering upon sample level clustering engages less differentially private noise to achieve the same level of privacy compared to sample level clustering.
  \item We conducted extensive experiments to investigate the effects of the parameters and the effectiveness of MC-GEN on three classification datasets. Meanwhile, we compared MC-GEN with three existing methods to demonstrate its utility. 
\end{itemize}

The rest of the paper is organized as follows: Section \ref{sec:Related work} introduces some related works. Section \ref{sec:Preliminary} presents some preliminary of our approach. Section \ref{sec:Methodology} describes our methodology and generative model of MC-GEN. Section \ref{sec: Experimental Evaluation} presents the experimental evaluation of MC-GEN. Section \ref{sec:Conclusion} makes the conclusions.

\section{Related Work}
\label{sec:Related work}
Moritz et al. \cite{hardt2012simple} proposed an algorithm that combined the multiplicative weighs approach and exponential mechanism for differentially private database release. It finds and improves an approximation dataset to better reflect the original data distribution by using the multiplicative weighs update rule. The weight for each data record answer to desired queries in the approximation dataset will scale up or down depending on its contribution to the query result. The queries are sampled and measured by using the exponential mechanism and Laplace mechanism, which guarantee differential privacy. This work only considers the privacy solution to a set of linear queries.

Some of the research works applied the conditional probabilistic model to synthetic data generation. A new privacy notion, `` Plausible Deniability'' \cite{bindschaedler2017plausible}, has been proposed and achieved by applying a privacy test after generating the synthetic data. The generative model proposed in this paper is a probabilistic model which captures the joint distribution of features based on correlation-based feature selection (CFS)\cite{hall1999correlation}. The original data is transformed into synthetic data by using conditional probabilities in a probabilistic model. This work also proved that by adding some randomizing function to `` Plausible Deniability'' can guarantee differential privacy. The synthetic data generated by this approach contains a portion of the original data, which may lead to a potential privacy issue. PrivBayes \cite{zhang2017privbayes} generated the synthetic data by releasing a private Bayesian network. Starting from a randomly selected feature node, it extends the network iteratively by selecting a new feature node from the parent set using the exponential mechanism. They also applied the Laplace mechanism on the conditional probability to achieve the private Bayesian network. Bayesian network is a good approach for discrete data. However, using the encoding method to represent the dataset containing many numerical data introduced more noise to synthetic datasets.

Some works form a generative model on preprocessed original data. A non-interactive private data release method has been proposed in \cite{chanyaswad2017coupling}. They projected the data into a lower dimension and proved it is nearly Gaussian distribution. After projecting the original data, the Gaussian mixture model (GMM) is used to model the original data based on the estimated statistical information. The synthetic data is generated by adding differential privacy noise on GMM parameters. However, there is some information loss while projection the data into a lower dimension. Josep Domingo-Ferrer et al. \cite{soria2017differentially} proposed the release of a differentially private dataset based on microaggregation\cite{domingo2005ordinal}, which reduced the noise required by differential privacy based on k$-$anonymity~\cite{sweeney2002k}. They clustered original data into clusters, and records in each cluster will be substituted by a representative record (mean) computed by each cluster. The synthetic data is generated by applying the Laplace mechanism to representative records. They also mentioned the idea to consider the feature relationship while clustering, but the paper does not come up with a detailed methodology. Using the representative records as seed data to generate synthetic data may lose some variance in the original datasets, which makes the synthetic data not accurate.

DPGAN~\cite{https://doi.org/10.48550/arxiv.1802.06739}, BGAN~\cite{9522203} and PATE-GAN~\cite{yoon2018pategan} proposed differential private synthetic data release method based on generative adversarial network (GAN)~\cite{goodfellow2014generative}. Differentially private generative adversarial network (DPGAN) add noise on the gradient and clip the weight during the training process based on Wasserstein GAN to guarantee the privacy. PATE-GAN proposed private aggregation of teacher (PATE) ensembles teacher-student framework to generate synthetic data. It initialized k teacher models and k data subsets, and each teacher is trained to discriminate between the original and fake data using the corresponding subset. The student model is trained on the data label by ensemble results with the noise of teacher models. Then, the data generator is trained to fool the student model to get the synthetic data. However, using deep learning structures would lead to high demand for original data and a time-consuming training process. A recent benchmarking of differential private synthetic data generation~\cite{abs-2112-09238} shows that the GAN-based methods are not good at preserving the critical statistical information from original data distributions. This paper focuses on non-image data and how to preserve the statistical characteristic. Thus, the GAN-based mechanisms are not considered in the comparison.  

\section{Preliminary}
\label{sec:Preliminary}
This section introduces the background knowledge of this paper.

\subsection{Differential Privacy}
Differential privacy proposed by D.work et al.~\cite{dwork2011differential} is one of the most popular privacy definitions which guarantee strong privacy protection on individual samples. The formal definition of $\epsilon$- differential privacy for synthetic datasets generation is defined as below:

\begin{mydef} (Differential Privacy)

A randomized mechanism $M$ provides $\epsilon$ - differential privacy, if for any pair of datasets $D_{1}$, $D_{2}$ that differ in a single record (neighboring datasets) and for any possible synthetic data outputs S $\subset$ Range(M), it satisfies

\begin{equation}\label{DP}
 \frac{Pr(M(D_{1}) =  S)}{Pr(M(D_{2}) = S)} \leq exp(\epsilon).
\end{equation}

\end{mydef}

The privacy parameter (privacy budget) $\epsilon$ is used to present the privacy level achieved by the randomized mechanisms $M$. The privacy budget ranges from 0 to $\infty$. If $\epsilon$ = 0, the probability of outputting synthetic dataset $S$ for $D_{1}$ and $D_{2}$ are exactly the same. It ensures the perfect privacy guarantee, which means $D_{1}$, $D_{2}$ can not be distinguished by observing the synthetic datasets; if $\epsilon$ = $\infty$, the synthetic datasets generated from $D_{1}$ and $D_{2}$ are always look different (like original datasets), there is no privacy guarantee at all. If differential privacy has been applied to synthetic datasets, it is hard to distinguish whether a specific individual is in the original dataset by observing synthetic datasets. In other words, differential privacy makes the synthetic dataset plausible. For example, there is a patient record dataset and a new patient record (privacy information) has been added recently. After publishing the differential private synthetic dataset, the data users can not infer the information of the new patient. It is because the original dataset with ($D_{1}$) or without ($D_{2}$) the new patient are most likely to generate the synthetic dataset($S$).

In general, a differential privacy mechanism can be achieved by adding noise to the output (synthetic dataset) with a privacy parameter between 0 and 1. The mechanism used to achieve $\epsilon$ - differential privacy is called differentially private sanitizer and it is always associated with the sensitivity, also known as $L_{1}$ - Sensitivity, defined as follows:

\begin{mydef} ($L_{1}$ - Sensitivity)
For a function $f: D^{n} \to \mathbb{R}^{d}$ which maps datasets to real number domain, the sensitivity of the function f for all neighboring datasets pairs $D$, $D'$ is defined as follows:

\begin{equation}\label{L1sensitivity}
\Delta D =\max_{D,D'}||f(D)-f(D')||.
\end{equation}
\end{mydef}

Laplace mechanism achieves $\epsilon$-differential privacy by adding noise drawn from Laplace distribution~\cite{dwork2006calibrating} to the output.

\begin{mydef} (Laplace Mechanism)
For a function $f: D^{n*d} \to \mathbb{R}^{n*d}$ which maps datasets to real number domain, the mechanism M defined in the following equation provides $\epsilon$-differential privacy:

\begin{equation}\label{Laplace Mechanism}
  M(D) = f(D) + Laplace(\Delta f/ \epsilon).
\end{equation}

\end{mydef}

\subsection{Agglomerative Hierarchical Clustering}An essential part of our methodology is hierarchical clustering~\cite{johnson1967hierarchical}. Hierarchical clustering is an algorithm that clusters input samples into different clusters based on the proximity matrix of samples. The proximity matrix contains the distance between each cluster. Agglomerative hierarchical clustering is a bottom-up approach. It starts by assigning each data sample to its own group and merges the pairs of clusters that have the smallest distance to move up until there is only a single cluster left.

\subsection{Microaggregation}
Microaggregation~\cite{domingo2005ordinal} is a kind of dataset anonymization algorithm that can achieve k-anonymity. There are serval settings for microaggregation, notice, the microaggregation mentioned here is a simple heuristic method called maximum distance to average record (MDAV) proposed by Domingo-Ferrer et al.~\cite{domingo2005ordinal}. MDAV clusters samples into clusters, in which each cluster contains exactly k records, except the last one. Records in the same cluster are supposed to be as similar as possible in terms of distance. Furthermore, each record in the cluster will be replaced by a representative record of the cluster to complete the data anonymization.

\subsection{Multivariate Gaussian Generative Model}
The multivariate Gaussian generative model keeps the multivariate Gaussian distribution~\cite{ahrendt2005multivariate}, which is parameterized by the mean $\mu$ and covariance matrix $\Sigma$. Formally, the density function of multivariate Gaussian distribution is given by:
	
	\begin{equation}\label{MGD}
	f(x)= \frac{1}{(2\pi)^{n/2}|\Sigma|^{1/2}} exp(-\frac{1}{2}(x-\mu)^{T}\Sigma^{-1}(x-\mu)).
	\end{equation}
	
Data samples drawn from the Multivariate Gaussian generative model are under Multivariate Gaussian distribution.

\section{Methodology}
\label{sec:Methodology}
In this section, we present the methodology of our approach.
	
\subsection{Problem Statement}
Given a numerical dataset $D^{n \times d}$ ($n$ samples and $d$ features) which contains sensitive privacy records. The data owner expects to share dataset $D$ to an untrust party in a secure way. In this paper, we aim to use the synthetic dataset $D'$ to substitute the original dataset $D$ in data sharing to prevent privacy leakage. The synthetic dataset $D'$ not only maintains some certain information in original datasets but also protected by a certain level of privacy. 

\begin{figure}[h]
	\centering
	\includegraphics[scale=0.4]{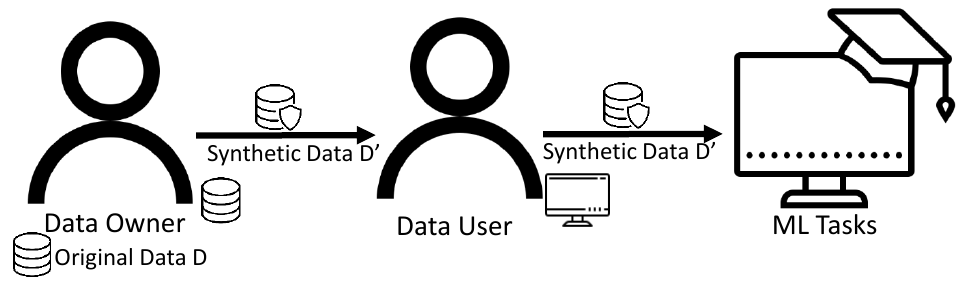}
	\caption{Synthetic Data Use Case}\label{usecase}
\end{figure}

\subsection{Methodology Overview}
Fig. \ref{flowchart} illustrates the design of our approach. It includes four processes worked collectively to generate the synthetic dataset which satisfies differential privacy:
\begin{itemize}
\item Data preprocessing: Combining independent feature sets and microaggregation~\cite{domingo2005ordinal} (multi-level clustering) to produce data clusters.

\item Statistic extraction: Extract the representative statistical information from each data cluster.
\item Differential privacy sanitizer: Introduce differential private noise on extracted statistical information.
\item Data generator: Generate synthetic data sample by sample from the noisy generative model.
\end{itemize}

\begin{figure*}[h]
	\centering
	\includegraphics[scale=0.6]{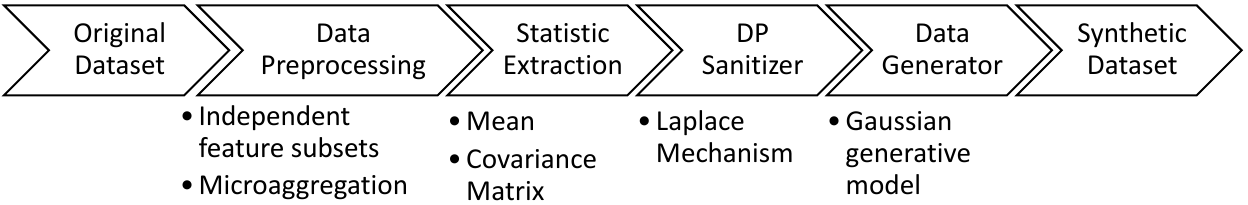}
	\caption{MC-GEN Algorithm}\label{flowchart}
\end{figure*}

We will explain each component in our design in the following sections.

\subsection{Data Preprocessing}
The conventional way to generate data clusters only groups the data at the sample level. For example, microaggregation (MDAV)~\cite{domingo2005ordinal} cluster the data in full feature dimension and add the differentially private noise on the representative records. Two kinds of errors may be introduced to jeopardize the utility of this process:

\begin{itemize}
  \item The false feature correlation introduced by sample level clustering. When modeling the output clusters from sample level clustering, some clusters may carry some correlation which not exist while looking into all data samples. This false feature correlation may apply unnecessary constraints when modeling the data clusters, which may lead the synthetic data into a different shape. \\
  \item  The noise variance introduced by DP mechanism. Intuitively, the less noise introduced from differential privacy results in higher utility. Hence, reducing the noise from DP mechanism also helps us to improve the data utility.
\end{itemize}

To smooth these two errors, we design multi-level clustering that combines independent feature sets (IFS) with microaggregation in our approach consecutively. Namely, we not only cluster the data at the sample level but also the feature level. Feature level clustering helps the generative model to capture the correct correlation of features. Compared to sample level clustering, using multi-level clustering also reduced the total noise variance introduced to the synthetic datasets, detailed proof shown in section \ref{Statistic Extraction and Differentially Private Sanitizer}. 
	
\subsubsection{Feature Level Clustering}

Given a numerical dataset $D^{n \times d}$, we divided data features into m independent feature sets using agglomerative hierarchical clustering. A distance function that converts Pearson correlation to distance has been designed to form the proximity matrix in hierarchical clustering. Features that have a higher correlation should have a lower distance and a lower correlation corresponding to the higher distance. This approach results in that features in the same set are more correlated to each other and less correlated to the features in other feature sets. However, hierarchical clustering needs to specify the number of clusters to be divided. We use Davies-Bouldin~\cite{davies1979cluster} index to choose the best split from all possible numbers of clusters. The distance function in the proximity matrix is shown below:

\begin{equation}\label{eq:corrtodis}
  d = 2n[1-Corr(X,Y)],
\end{equation}
where Corr(X, Y) is the Pearson correlation between two random variables, e.g., feature pair in original datasets.

\subsubsection{Sample Level Clustering}
Based on the output of feature level clustering, we applied microaggregation~\cite{domingo2005ordinal} on each independent feature set (IFS). The purpose of microaggregation is to assign the homogeneous samples to the same cluster, which preserves more information from the original data. On the other hand, referring to~\cite{domingo2005ordinal}, sensitivity on each sample cluster can be potentially reduced compared to the global sensitivity. This reduction will result in involving less noise in the differential privacy mechanism. In other words, it enhances the data utility under the same level of privacy guarantee.

\begin{figure}[H]
	\centering
	\includegraphics[scale=0.7]{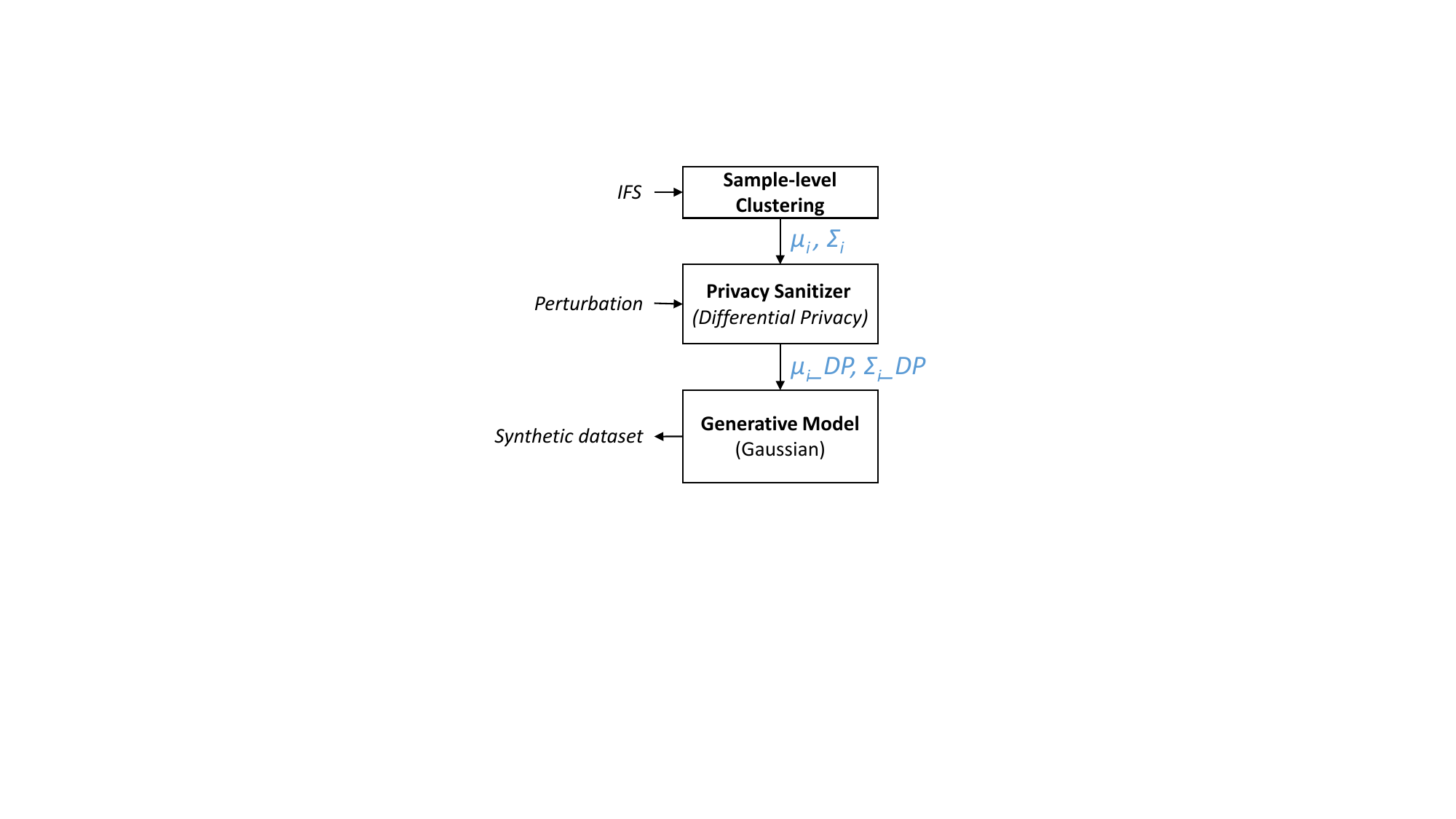}
	\caption{Synthetic Data Generator}\label{syngenerator}
\end{figure}

\subsection{Statistic Extraction and Privacy Sanitizer}\label{Statistic Extraction and Differentially Private Sanitizer}

\subsubsection{Statistic Extraction}
For a given dataset $D$, features are divided into $m$ IFSs on the feature level. In each IFS, the data with selected features is clustered into $j$ clusters by microaggregation on the sample level. The combination of feature and sample level clustering is called multi-level clustering. It outputs $m \times j$ clusters in which each cluster has at least k records. The statistical information is extracted from each cluster to generate the synthetic data.

\subsubsection{Privacy Sanitizer}
Assuming each cluster forms a multivariate Gaussian distribution, the mean ($\mu_{i}$) and covariance matrix ($\Sigma_{i}$) are computed for each cluster $c_{i}$. To ensure differential privacy, the generative model is built based on $\mu_{i}\_DP$ and $\Sigma_{i}\_DP$ which are achieved by the privacy sanitizer, as shown in Fig. \ref{syngenerator}. The privacy sanitizer adds differentially private noise on $\mu_{i}$ and $\Sigma_{i}$.  Algorithm \ref{alg:MC-GEN} illustrates the process of privacy sanitizer. $n_{mean}$ and $n_{covarM}$ denote the noise perturbation on mean and covariance matrix.

\begin{algorithm}[h]
\caption{MC-GEN privacy sanitizer}\label{alg:MC-GEN}
\LinesNumbered
\KwIn{$\mu_{i},\Sigma_{i}$}
\KwOut{$\mu_{i}\_DP,\Sigma_{i}\_DP$}

\For {each cluster:}{
    $\mu_{i}, \Sigma_{i} \gets c_{i}$\;
    $n_{mean} \gets $Draw $d$ samples from $Lap(0,\frac{\Delta}{\lvert c_{i} \rvert\epsilon_{m}})$\;
    $n_{covarM} \gets $Draw $\frac{d^{2}+d}{2}$ samples from $Lap(0,\frac{\Delta}{\lvert c_{i} \rvert\epsilon_{m}})$\;
    $\mu_{i}\_DP = \mu_{i} + n_{mean} $ \;
    $ \Sigma_{i}\_DP = \Sigma_{i} + n_{covarM} $\;
    \Return $\mu_{i}\_DP,\Sigma_{i}\_DP$
}

\end{algorithm}

%

The differentially private noise added on $\mu_{i}$ and $\Sigma_{i}$ suppose to be as little as possible to preserve the utility of synthetic data. It refers to the generative model captures the statistical information more precisely. Thus, it is very critical to investigate and control the noise variance introduced by privacy sanitizer. A contribution of MC-GEN is applying multi-level clustering reduces the overall noise introduced from differential privacy mechanism compared to sample level clustering. The following proof demonstrates the noise variance of multi-level clustering.
\begin{thm}
The noise variance introduced by multi-level clustering on mean vector $\mu$ is $\sum\limits_{m=1}^{m}\sum\limits_{j=1}^{j}\frac{\Delta IFS_{m}}{\lvert  C_{j}^{m} \rvert \epsilon}d$. 
\end{thm}

\begin{proof}

If multi-level clustering is applied, dataset D has been vertically partitioned into $m$ IFSs with corresponding data, and data in each IFS has been clustered into $j$ clusters by microaggregation. 

Let $\Delta IFS_{m}$ denotes the $L_{1}$ - sensitivity of $m_{th}$ IFS, $C_{j}^{m}$ denotes the $j_{th}$ cluster in $IFS_m$, $d_{m}$ denotes the size of $IFS_m$, and $d$ denotes the total number of features. The noise variance $N_{IFS_{m}}$ of $m_{th}$ IFS is the sum of noise variances on each cluster in this IFS. Noise on a single cluster by microaggragation has been shown in ~\cite{soria2017differentially}. Thus, noise variance $N_{IFS_{m}}$ can be written as:

\begin{equation}\label{mean_IFS_detail}
  N_{IFS_{m}} =\sum\limits_{j=1}^{j}\frac{\Delta IFS_{m}}{\lvert  C_{j}^{m} \rvert \epsilon_{m}}d_{m},
\end{equation}

\noindent where the $\epsilon$ has been divided proportionally based on the size of each IFS:

\begin{equation}
\left\{
\begin{aligned}\label{epsilon_description}
  \epsilon_{m} &=  \frac{d_{m}}{d}\epsilon \\
  \epsilon &= \epsilon_{1}+ \epsilon_{2}+ ... + \epsilon_{m}.
\end{aligned}
\right.
\end{equation}

The noise on each IFS is independent, and the total noise variance on the mean vector by multi-level clustering $N_{1}$ is the sum of noise on all IFSs.

\begin{equation}\label{mean_IFS_Sum}
  N_{1} = \sum\limits_{m=1}^{m}N_{IFS_{m}}=N_{IFS_{1}} + N_{IFS_{2}} + ... + N_{IFS_{m}}\\
\end{equation}

Based on equation (\ref{mean_IFS_detail}) and equation (\ref{epsilon_description}), use addition commutative to rewrite equation (\ref{mean_IFS_Sum}) as following:

\begin{equation}\label{mean_IFS_Sum_addcom}
 N_{1} = \sum\limits_{m=1}^{m}\sum\limits_{j=1}^{j}\frac{\Delta IFS_{m}}{\lvert  C_{j}^{m} \rvert \epsilon}d\\.
\end{equation}
 
\end{proof}

\begin{thm}
The noise variance introduced by multi-level clustering on covariance matrix $\Sigma$ is $\sum\limits_{j=1}^{j}\sum\limits_{m=1}^{m}\frac{\Delta IFS_{m}}{\lvert  C_{j}^{m} \rvert \epsilon}d_{m}^{2}$.
\end{thm}

\begin{proof}

 \noindent Use the same notation in Theorem 1 to describe the noise variance on the covariance matrix:
 
\begin{equation}\label{covar_IFS_detail}
  N_{IFS_{m}} =\sum\limits_{j=1}^{j}\frac{\Delta IFS_{m}}{\lvert  C_{j}^{m} \rvert \epsilon_{m}}d_{m}^{2}.
\end{equation}

The noise on each IFS is independent, and the total noise variance on the mean vector by multi-level clustering $N_{2}$ is the sum of noise on all IFSs.

\begin{equation}\label{covar_IFS_Sum}
  N_{2} = \sum\limits_{m=1}^{m}N_{IFS_{m}}=N_{IFS_{1}} + N_{IFS_{2}} + ... + N_{IFS_{m}}\\
\end{equation}

Based on equation (\ref{epsilon_description}) and (\ref{covar_IFS_detail}), rewrite equation (\ref{covar_IFS_Sum}):
\begin{equation}\label{covar_IFS_Sum_addcom}
  N_{2} = \sum\limits_{m=1}^{m}\sum\limits_{j=1}^{j}\frac{\Delta IFS_{m}}{\lvert  C_{j}^{m} \rvert \epsilon}d_{m}^{2}\\.
\end{equation}

\end{proof}

The following proofs demonstrates the noise variance of sample level clustering. 
 
\begin{thm}
The noise variance introduced by sample level clustering on mean vector $\mu$ is $\sum\limits_{j=1}^{j}\frac{\Delta D}{\lvert  C_{j} \rvert \epsilon}d$.
\end{thm}

\begin{proof}

If sample level clustering is applied, dataset D has been clustered into $j$ clusters. Let $\Delta D$ denote the $L_{1}$ - sensitivity of dataset D, and $C_{j}$ denote the $j_{th}$ cluster. Refer to~\cite{soria2017differentially}, the total noise variance on the mean vector by sample level clustering $N_{3}$ can be written as follows:

\begin{equation}\label{mean_noIFS}
  N_{3}=\frac{\Delta D}{\lvert  C_{1} \rvert \epsilon}d +  \frac{\Delta D}{\lvert  C_{2} \rvert \epsilon}d + ... + \frac{\Delta D}{\lvert  C_{j} \rvert \epsilon}d =\sum\limits_{j=1}^{j}\frac{\Delta D}{\lvert  C_{j} \rvert \epsilon}d.
\end{equation}

\end{proof}

\begin{thm}
The noise variance introduced by sample level clustering on covariance matrix $\Sigma$ is $\sum\limits_{j=1}^{j}\frac{\Delta D}{\lvert  C_{j} \rvert \epsilon}d^{2}$.
\end{thm}

\begin{proof}

Using the same notation in Theorem 3, the total noise variance on the covariance matrix by sample level clustering $N_{4}$ can be written as follows:
\begin{equation}\label{covar_noIFS}
  N_{4} =\frac{\Delta D}{\lvert  C_{1} \rvert \epsilon}d^{2} +  \frac{\Delta D}{\lvert  C_{2} \rvert \epsilon}d^{2} + ... + \frac{\Delta D}{\lvert  C_{j} \rvert \epsilon}d^{2} =\sum\limits_{j=1}^{j}\frac{\Delta D}{\lvert  C_{j} \rvert \epsilon}d^{2}. 
\end{equation}

\end{proof}

\begin{prop}
The overall noise introduced by multi-level clustering $N_{MC}$ is less than the overall noise introduced by sample clustering $N_{SC}$.
\end{prop}

\begin{proof}
Since the differentially private noise is applied to the mean and covariance matrix respectively. The overall noise introduced on multi-level clustering can be written as:
\begin{equation}
     N_{MC} = N_{1}+ N_{2}.
\end{equation}
The overall noise introduced on sample level clustering can be written as:
\begin{equation}
     N_{SC} = N_{3}+ N_{4}.
\end{equation}

Since data in each IFS has been clustered by microaggregation with the same cluster size, which means:
\begin{equation}\label{equalsizeproof}
  C_{j}^{m} =  C_{j}^{m+1}.
\end{equation}

and each cluster j is a subset of the original dataset:
\begin{equation}\label{clusterSubsetproof}
  d^{m} < d.
\end{equation}

Equation (\ref{mean_IFS_Sum_addcom}) can be rewritten as:
\begin{equation}\label{mean_compare}
 \begin{aligned}
 N_{1}&=\sum\limits_{m=1}^{m}\sum\limits_{j=1}^{j}\frac{\Delta IFS_{m}}{\lvert  C_{j}^{m} \rvert \epsilon}d\\
        &= \sum\limits_{j=1}^{j} \frac{\Delta IFS_{1}+\Delta IFS_{2}+...+\Delta IFS_{m}}{\lvert  C_{j} \rvert \epsilon}d\\
        &= \sum\limits_{j=1}^{j}\frac{\Delta D}{\lvert  C_{j} \rvert \epsilon}d = N_{3}.
 \end{aligned}
\end{equation}

Equation (\ref{covar_IFS_Sum_addcom}) can be rewritten as:
\begin{equation}\label{cov_compare}
\begin{aligned}
 N_{2}&= \sum\limits_{m=1}^{m}\sum\limits_{j=1}^{j}\frac{\Delta IFS_{m}}{\lvert  C_{j}^{m} \rvert \epsilon}d_{m}^{2}\\
        &= \sum\limits_{j=1}^{j} \frac{d^{2}_{1}\Delta IFS_{1}...+d^{2}_{m}\Delta IFS_{m}}{\lvert  C_{j} \rvert \epsilon}\\
        &< \sum\limits_{j=1}^{j}\frac{\Delta IFS_{1}+\Delta IFS_{2}+...+\Delta IFS_{m}}{\lvert  C_{j} \rvert \epsilon}d^{2} \\
        &=\sum\limits_{j=1}^{j}\frac{\Delta D}{\lvert  C_{j} \rvert \epsilon}d^{2} = N_{4}.
\end{aligned}
\end{equation}

Based on Eq.~(\ref{mean_compare}) and Eq.~(\ref{cov_compare}), we can know that $N_{MC}< N_{SC}$.
\end{proof}

Using feature level clustering not only mitigates the false feature correlation error but also helps to reduce the noise variance. Generally, the synthetic datasets generated with less noise will have better utility.

\subsection{Synthetic Data Generator}
The original multivariate Gaussian model is parameterized by mean ($\mu_{i}$) and covariance matrix ($\Sigma_{i}$). The algorithm \ref{alg:MC-GEN} outputs two parameters $\mu_{i}\_DP$ and $\Sigma_{i}\_DP$ that are protected by differential privacy. Hence the multivariate Gaussian model that is parameterized by $\mu_{i}\_DP$ and $\Sigma_{i}\_DP$ is also protected by differential privacy. Depending on the post-processing invariance of different privacy, all the synthetic data derived from differential private multivariate Gaussian models are protected by differential privacy as well. The synthetic data is synthesized sample by sample from the private multivariate Gaussian generative models with perturbed mean ($\mu_{i}\_DP$) and covariance matrix ($\Sigma_{i}\_DP$) based on Eq.~(\ref{MGD}). These multivariate Gaussian generative models are determined by the multiple-level clustering result of corresponding original data.

\section{Experimental Evaluation}
\label{sec: Experimental Evaluation}	
In this section, we show the experimental design and discuss the results of our approach.

\begin{table*}[t!]
\begin{center}
\begin{tabularx}{0.9\textwidth} { |>{\centering\arraybackslash}X || >{\centering\arraybackslash}X 
  | >{\centering\arraybackslash}X 
  | >{\centering\arraybackslash}X | }
 \hline
 \diagbox{Dataset}{Task}& SVM & Logistic Regression & Gradient Boosting\\
 \hline
 \multicolumn{4}{|c|}{Scenario 1} \\
 \hline
 Diabetes~\cite{nr} & 0.80 & 0.78 & 0.75\\
 \hline
 Adult~\cite{misc_adult_2} & 0.84 & 0.89 &0.91\\
 \hline
 Phishing~\cite{mohammad2012assessment} & 0.96 & 0.98 & 0.96\\
 \hline
 \multicolumn{4}{|c|}{Scenario 2} \\
 \hline
 Diabetes~\cite{nr} & 0.82 & 0.80 & 0.85\\
 \hline
 Adult~\cite{misc_adult_2} & 0.85 & 0.86 & 0.83\\
 \hline
 Phishing~\cite{misc_phishing_websites_327} & 0.96 & 0.94 & 0.94 \\
 \hline
\end{tabularx}
\caption{Baseline of Experiments}
\label{Baseline of Experiments}
\end{center}
\end{table*}

\subsection{Experiment Setting}
To evaluate the performance of the proposed method, we have implemented MC-GEN based on JAVA 8. We generated synthetic datasets and performed experiments under different cluster sizes ($k$) and different privacy parameters ($\epsilon$). The setting of $\epsilon$ varies from 0.1 to 1 and cluster size k is picked proportionally (20\%, 40\%, 60\%, 80\%, 100\%) based on the corresponding seed dataset. For each synthetic dataset, we evaluate the performance of three classification tasks: support vector machine (SVM), logistic regression, and gradient boosting. The classification models are implemented using python scikit-learn library~\cite{scikit-learn}\cite{sklearn_api}.  We also compared our method with other private synthetic data release methods. All the experiments are performed in two scenarios:


\begin{itemize}
   \item Original data training, synthetic data testing (Scenario 1):\\
  The classification algorithm trains on original data and tests on the synthetic data. For each experimental dataset, 20\% of the samples are used as seed dataset to generate the synthetic dataset, and 80\% is used as original data to train the model.
   \item Synthetic data training, original data testing (Scenario 2):\\
  The classification algorithm trains on synthetic data and tests on the original data. For each experimental dataset, 80\% of the samples are used as seed dataset to generate the synthetic dataset, and 20\% is used as original data to test the model.
\end{itemize}

\subsection{Experiment Dataset}
We conducted three datasets in our experiments from public sources (e.g. UCI repository~\cite{Dua:2019}, kaggle, libsvm datasets~\cite{chang2011libsvm}) to examine the performance of different classification algorithms. For dataset that contains categorical variables, we use one-hot encoding to convert it to numerical variables. All the features in the dataset are scaled to [-1,1]:

\subsubsection{Diabetes}
This dataset contains the diagnostic measurements of patient records. It has 768 samples and 8 features. The features include some patient information, such as blood pressure, BMI, insulin level, and age. The objective is to identify whether a patient has diabetes.
\subsubsection{Adult}
This dataset is extracted from the census bureau database. It has 48,842 samples and 14 features. The features include some information, such as age, work class, education, sex, and capital gain/loss. The objective is to predict the annual salary (over 50k or not) of each person.
\subsubsection{Phishing}
This dataset is used to predict phishing websites. It has 11,055 samples and 30 features. The features include information related address, HTML and JavaSrcipt, such as website forwarding, request URL, and domain registration length. 


%

\subsection{Experiment Setup and Metric}
In each round of experiments, we assume the data owners randomly select data from the original dataset as seed data. The size of seed data depends on the scenario. The seed data is input to MC-GEN to generate the synthetic dataset under different parameter combinations ($\epsilon$ and $k$). Then, the synthetic dataset is tested on three classification tasks to evaluate the performance. Accuracy is used as the evaluation metric, shown in Eq.~\ref{accMetric}. For each scenario on each dataset, we ran 100 rounds to get the average performance. 
\begin{equation}\label{accMetric}
    Accuracy = \frac{Number\ of\ correct\ predictions}{Total\ number\ of\ predictions}
\end{equation}

\begin{figure*}[!t]
\centering
	\begin{subfigure}{.32\textwidth}
		\includegraphics[width=\textwidth]{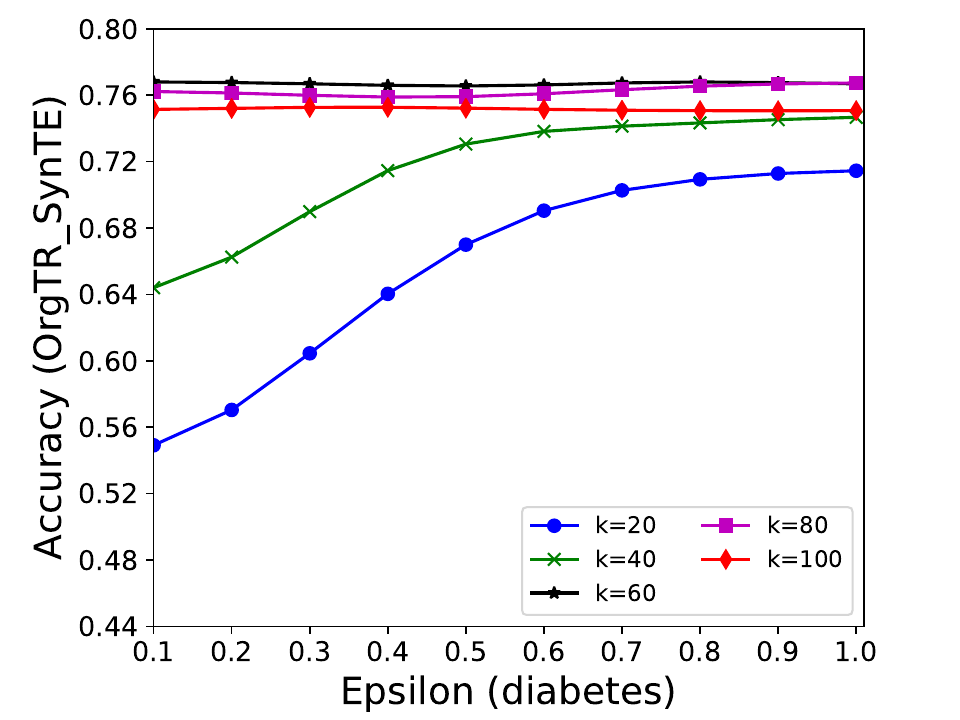}
		\caption{SVM}
        \label{diabetes_OrgTR_SVM_comparek}
	\end{subfigure}
	\begin{subfigure}{.32\textwidth}
		\includegraphics[width=\textwidth]{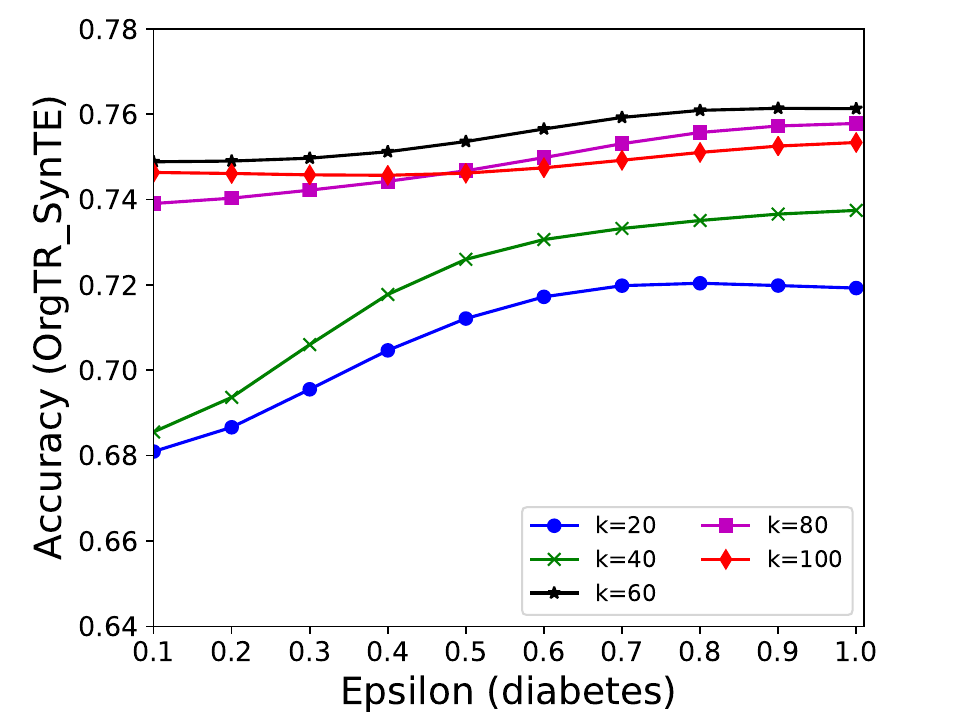}
		\caption{Logistic regression}
        \label{diabetes_OrgTR_LR_comparek}
	\end{subfigure}
	\begin{subfigure}{.32\textwidth}
		\includegraphics[width=\textwidth]{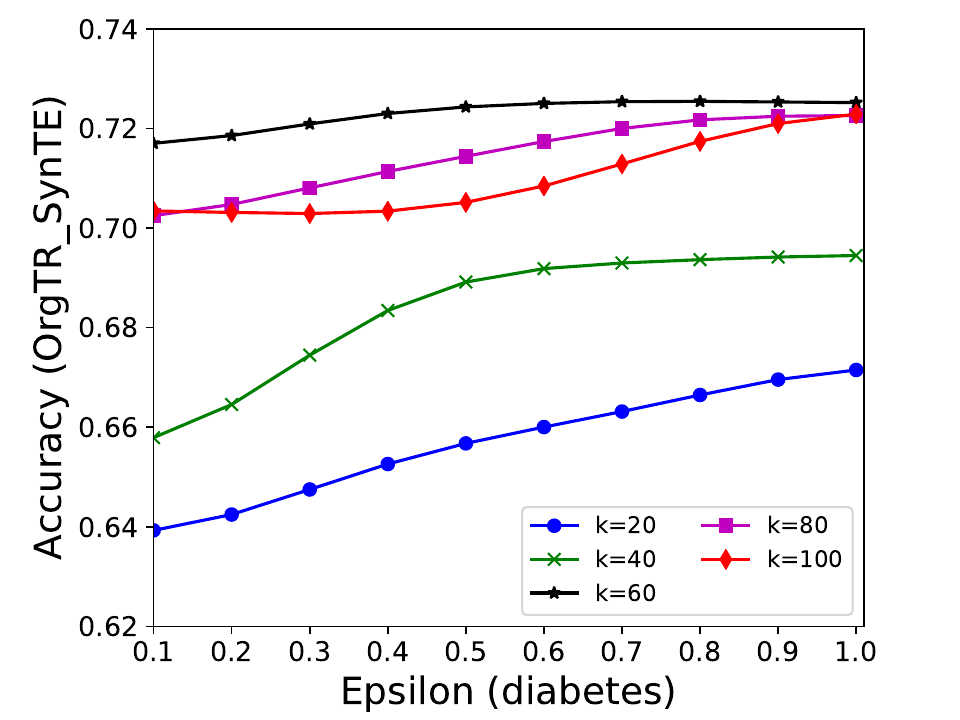}
		\caption{Gradient boosting}
        \label{diabetes_OrgTR_GB_comparek}
	\end{subfigure}
	\caption{Effect of cluster size (k) on diabetes dataset in scenario 1}
    \label{diabetes_comparek_orgTR}
\end{figure*}
%

\begin{figure*}[!t]
\centering
	\begin{subfigure}{.32\textwidth}
		\includegraphics[width=\textwidth]{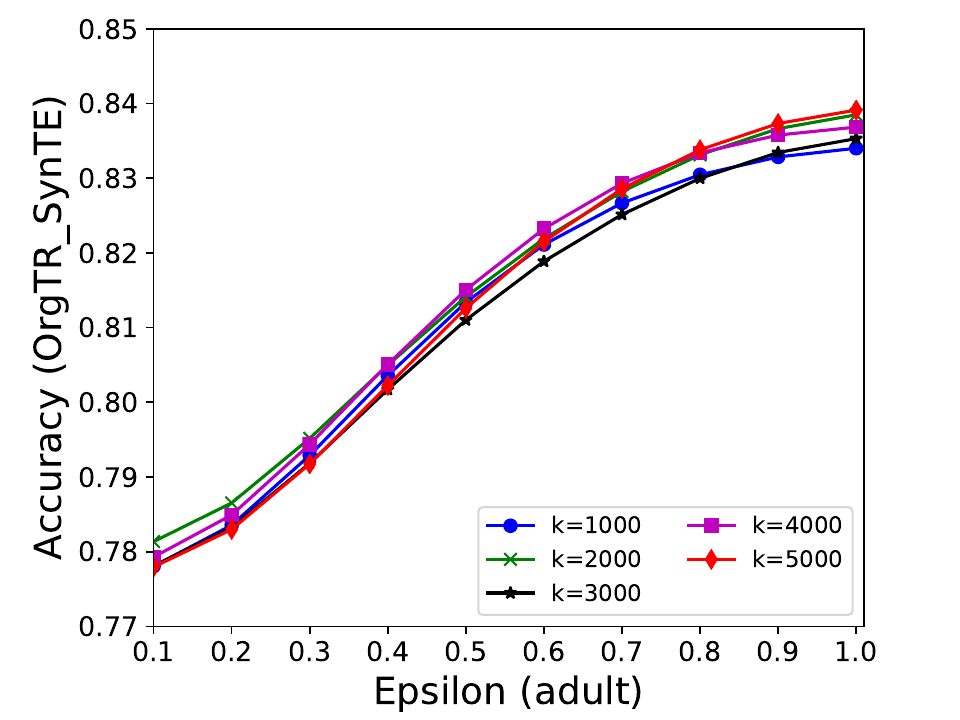}
		\caption{SVM}
        \label{adult9_OrgTR_SVM_comparek}
	\end{subfigure}
	\begin{subfigure}{.32\textwidth}
		\includegraphics[width=\textwidth]{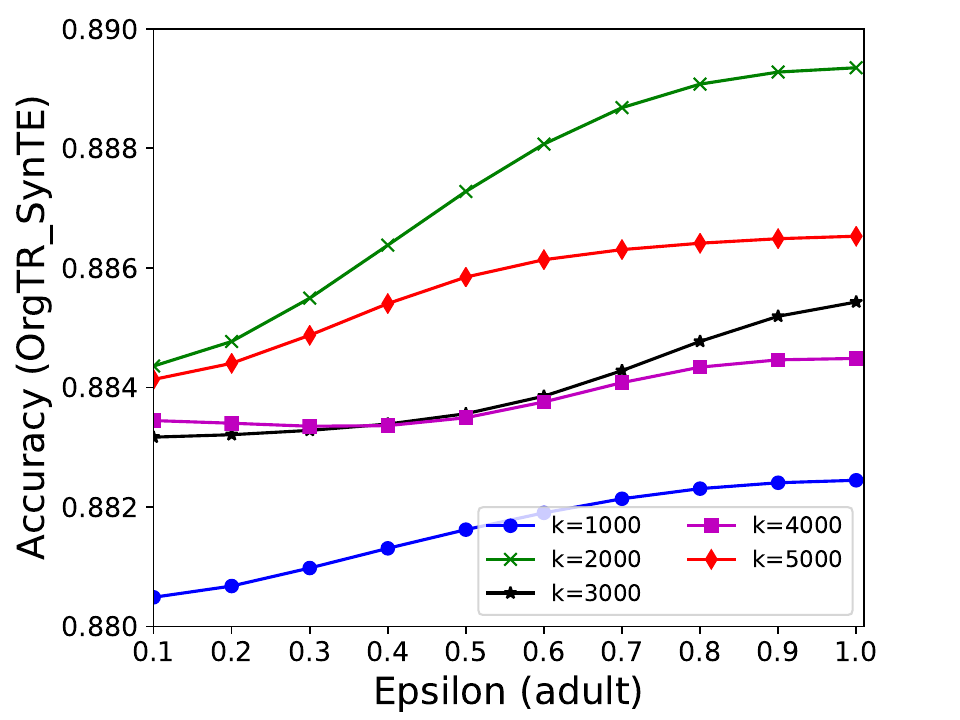}
		\caption{Logistic regression}
        \label{adult9_OrgTR_LR_comparek}
	\end{subfigure}
	\begin{subfigure}{.32\textwidth}
		\includegraphics[width=\textwidth]{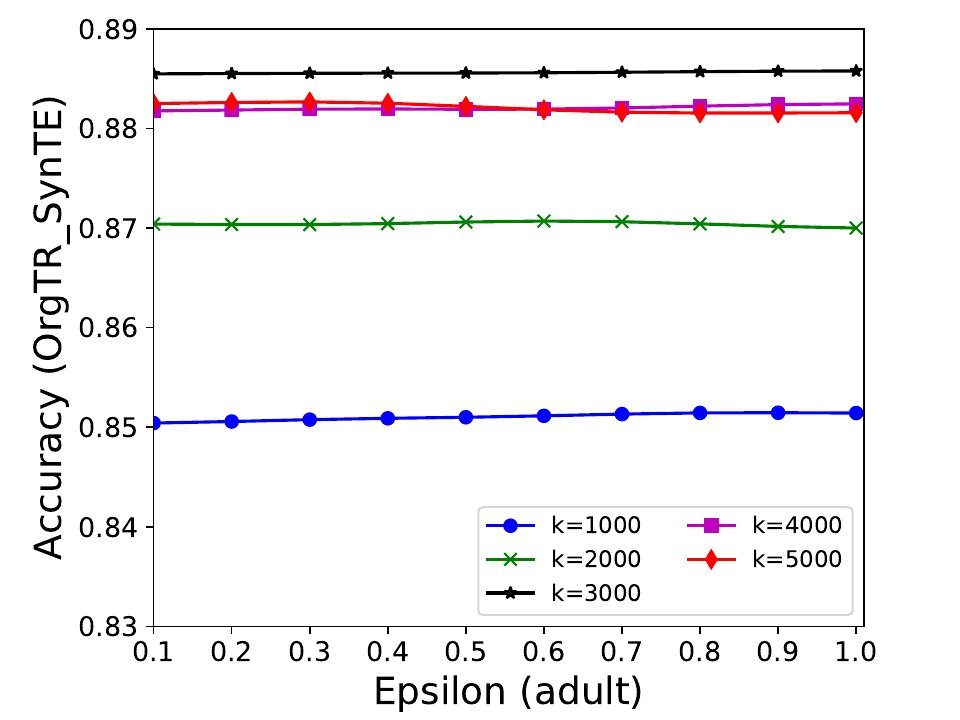}
		\caption{Gradient boosting}
        \label{adult9_OrgTR_GB_comparek}
	\end{subfigure}
	\caption{Effect of cluster size (k) on adult dataset in scenario 1}
    \label{adults_comparek_orgTR}
\end{figure*}
%

\begin{figure*}[!t]
\centering
	\begin{subfigure}{.32\textwidth}
		\includegraphics[width=\textwidth]{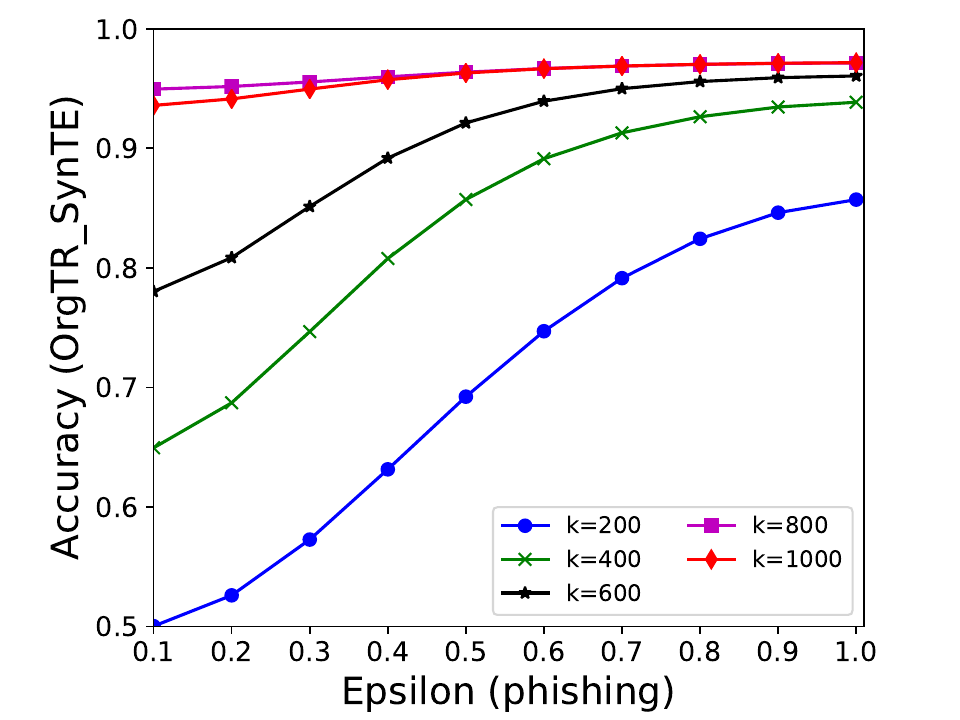}
		\caption{SVM}
        \label{phishing_OrgTR_SVM_comparek}
	\end{subfigure}
	\begin{subfigure}{.32\textwidth}
		\includegraphics[width=\textwidth]{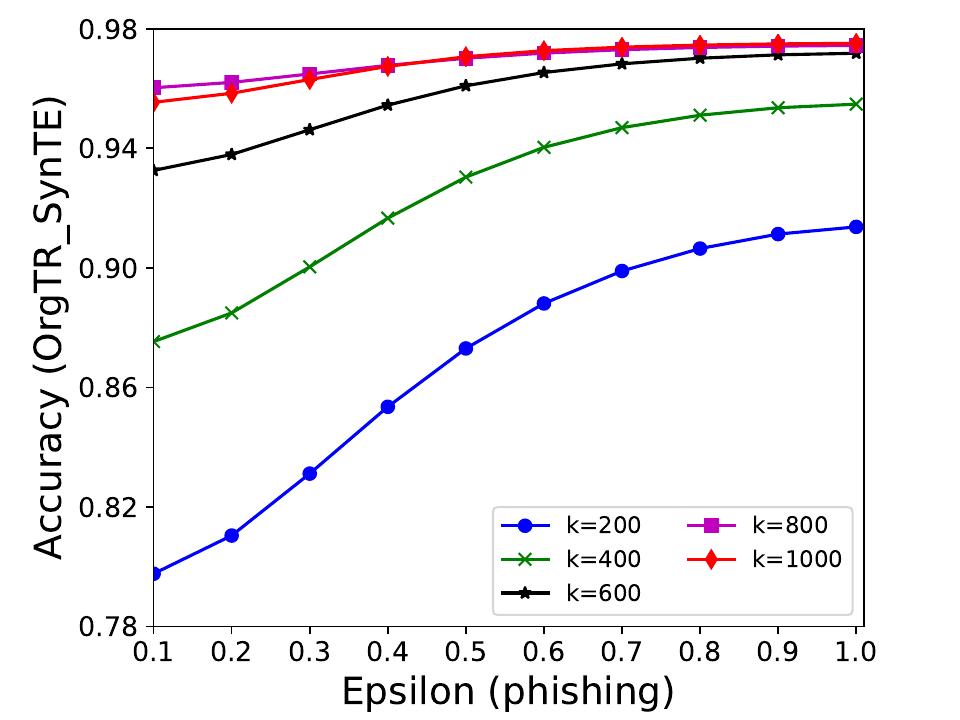}
		\caption{Logistic regression}
        \label{phishing_OrgTR_LR_comparek}
	\end{subfigure}
	\begin{subfigure}{.32\textwidth}
		\includegraphics[width=\textwidth]{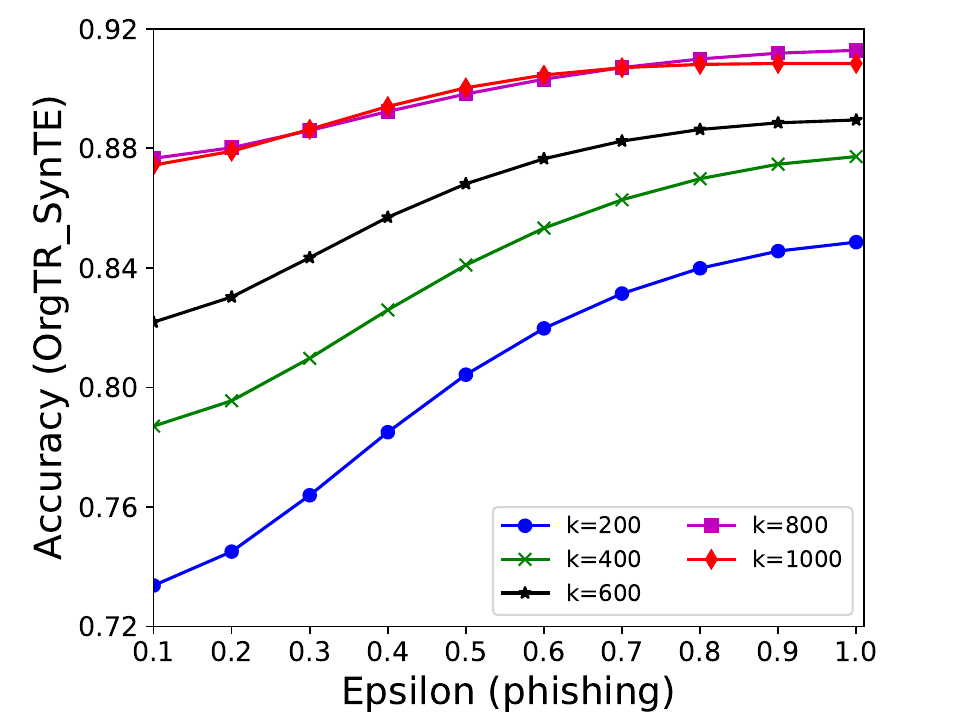}
		\caption{Gradient boosting}
        \label{phishing_OrgTR_GB_comparek}
	\end{subfigure}
	\caption{Effect of cluster size (k) on phishing dataset in scenario 1}
    \label{phishing_comparek_orgTR}
\end{figure*}
%

\begin{figure*}[!t]
\centering
	\begin{subfigure}{.32\textwidth}
		\includegraphics[width=\textwidth]{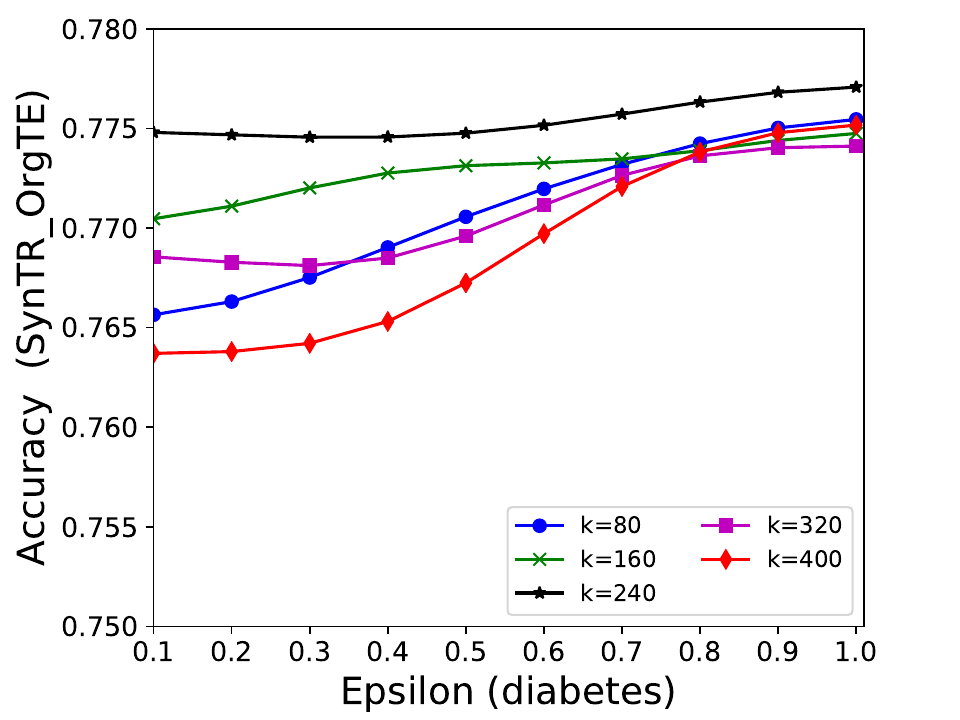}
		\caption{SVM}
        \label{diabetes_SynTR_SVM_comparek}
	\end{subfigure}
	\begin{subfigure}{.32\textwidth}
		\includegraphics[width=\textwidth]{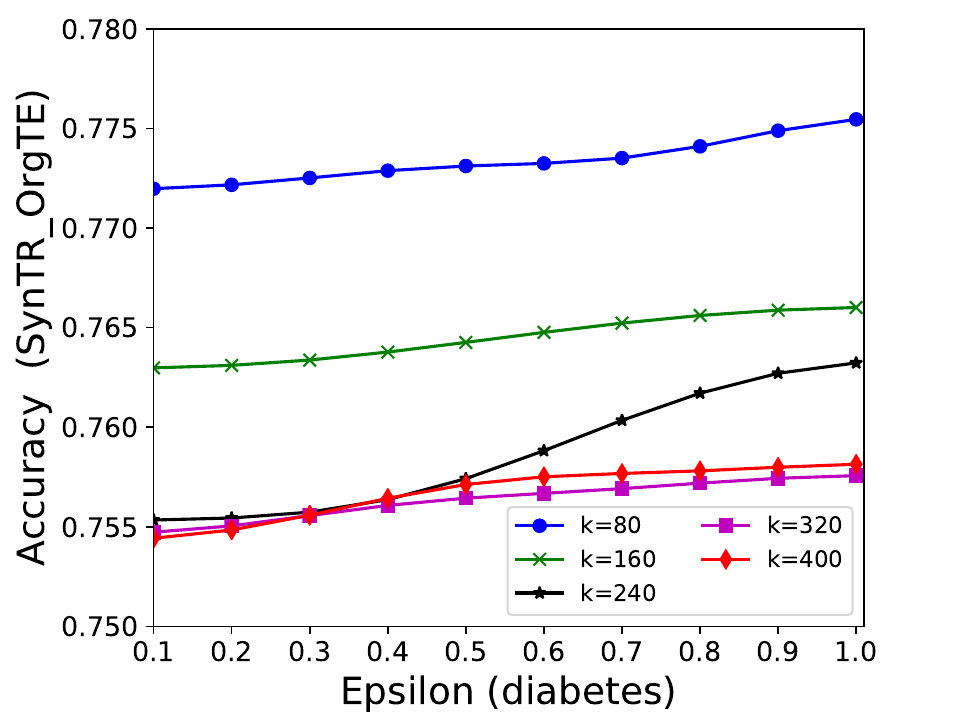}
		\caption{Logistic regression}
        \label{diabetes_SynTR_LR_comparek}
	\end{subfigure}
	\begin{subfigure}{.32\textwidth}
		\includegraphics[width=\textwidth]{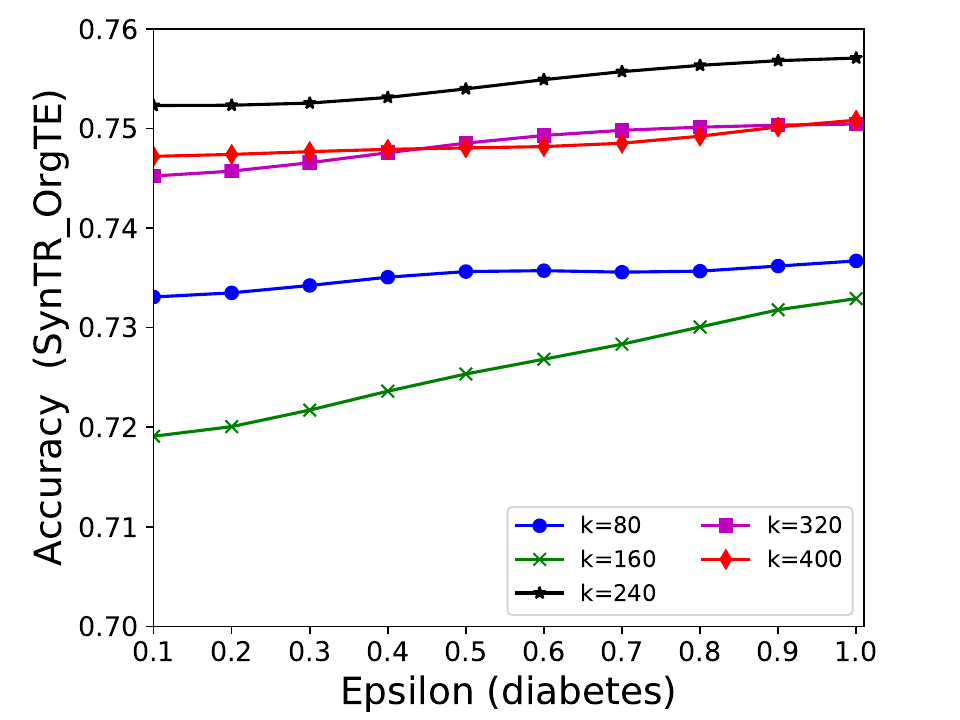}
		\caption{Gradient boosting}
        \label{diabetes_SynTR_GB_comparek}
	\end{subfigure}
	\caption{Effect of cluster size (k) on diabetes dataset in scenario 2}
    \label{diabetes_comparek_synTR}
\end{figure*}
%

\begin{figure*}[!t]
\centering
	\begin{subfigure}{.32\textwidth}
		\includegraphics[width=\textwidth]{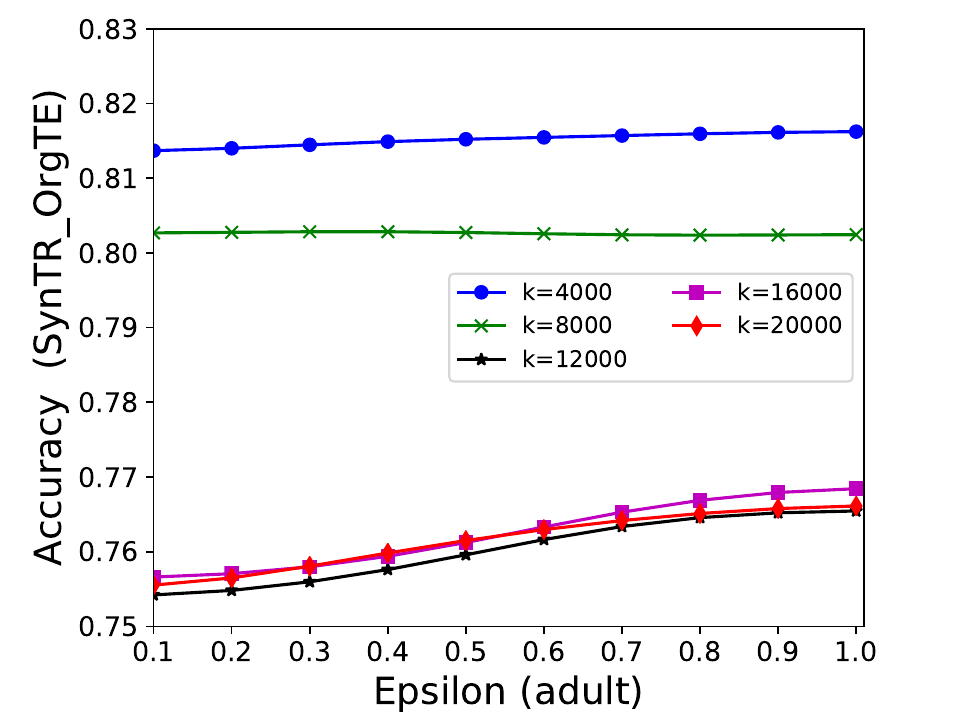}
		\caption{SVM}
        \label{adult9_SynTR_SVM_comparek}
	\end{subfigure}
	\begin{subfigure}{.32\textwidth}
		\includegraphics[width=\textwidth]{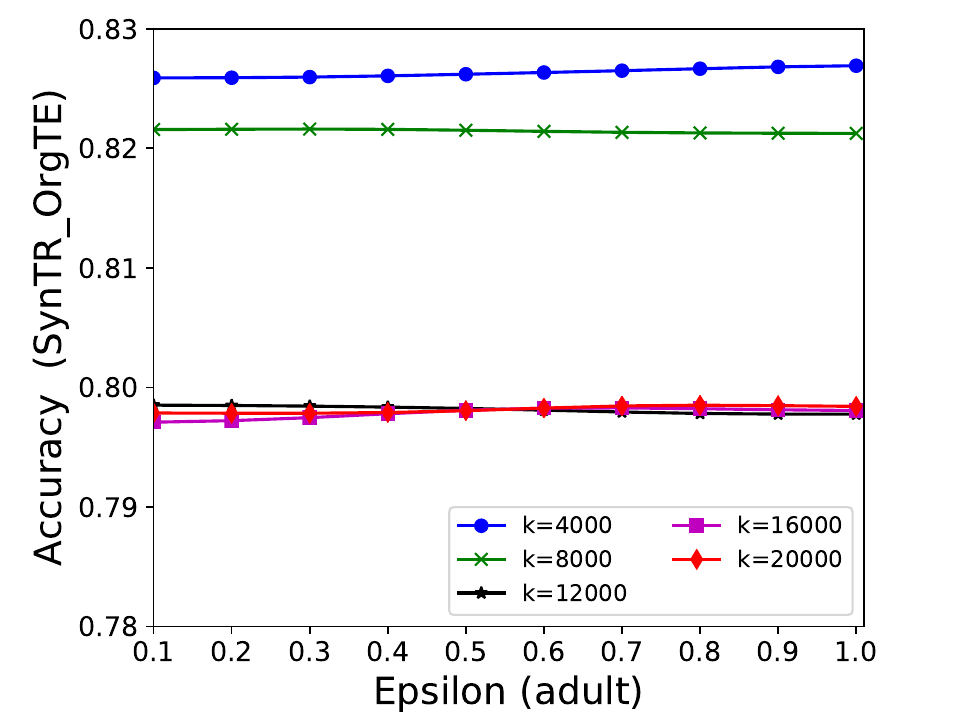}
		\caption{Logistic regression}
        \label{adult9_SynTR_LR_comparek}
	\end{subfigure}
	\begin{subfigure}{.32\textwidth}
		\includegraphics[width=\textwidth]{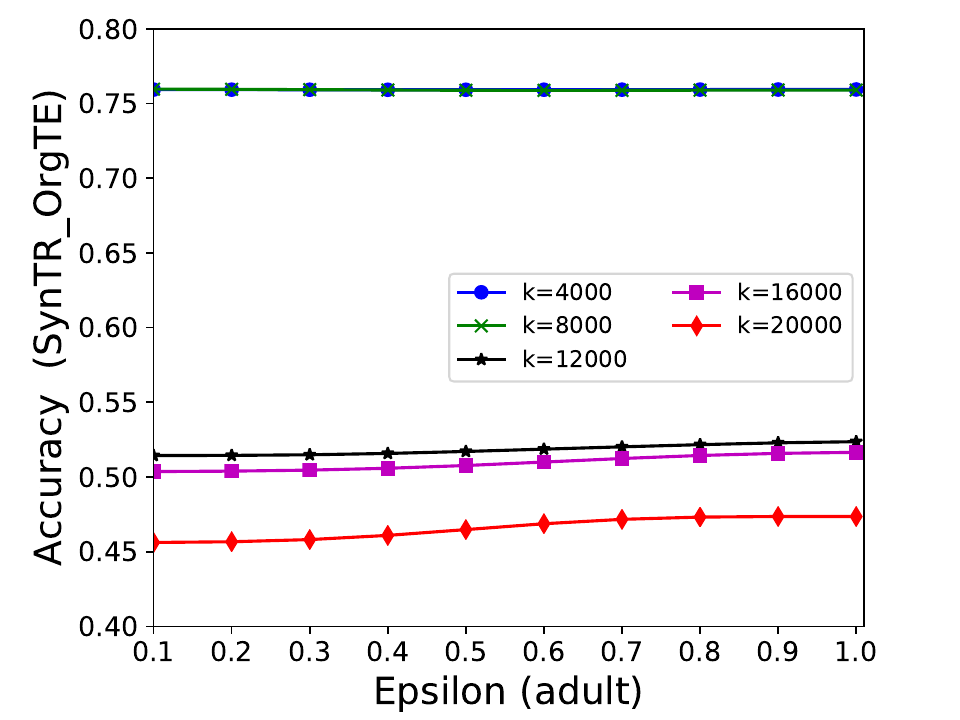}
		\caption{Gradient boosting}
        \label{adult9_SynTR_GB_comparek}
	\end{subfigure}
	\caption{Effect of cluster size (k) on adult dataset in scenario 2}
    \label{adults_comparek_synTR}
\end{figure*}
%

\begin{figure*}[!t]
\centering
	\begin{subfigure}{.32\textwidth}
		\includegraphics[width=\textwidth]{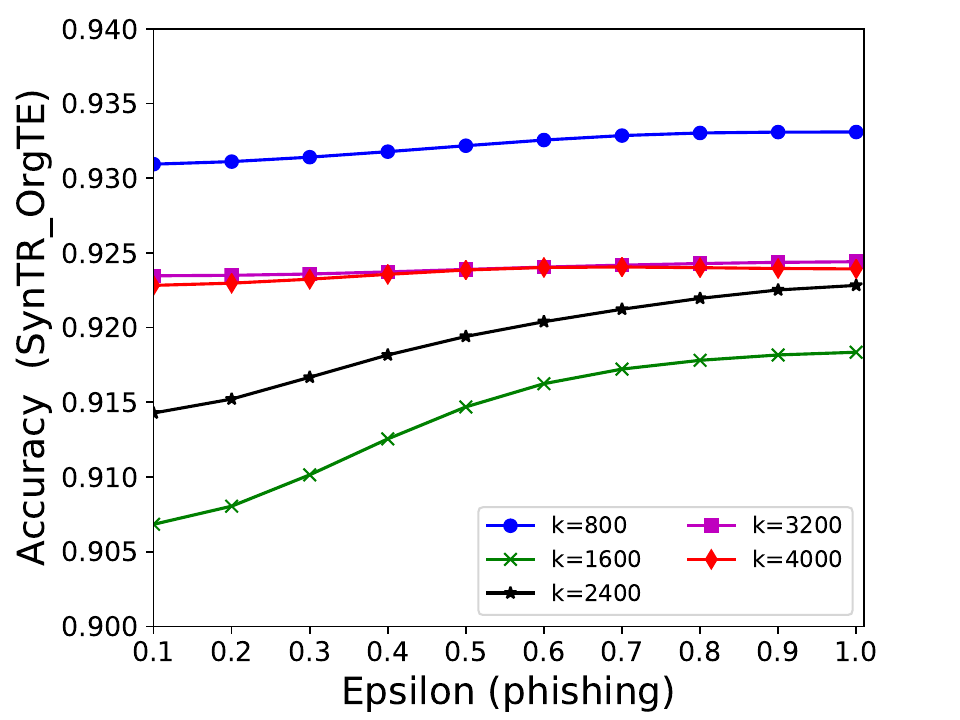}
		\caption{SVM}
        \label{phishing_SynTR_SVM_comparek}
	\end{subfigure}
	\begin{subfigure}{.32\textwidth}
		\includegraphics[width=\textwidth]{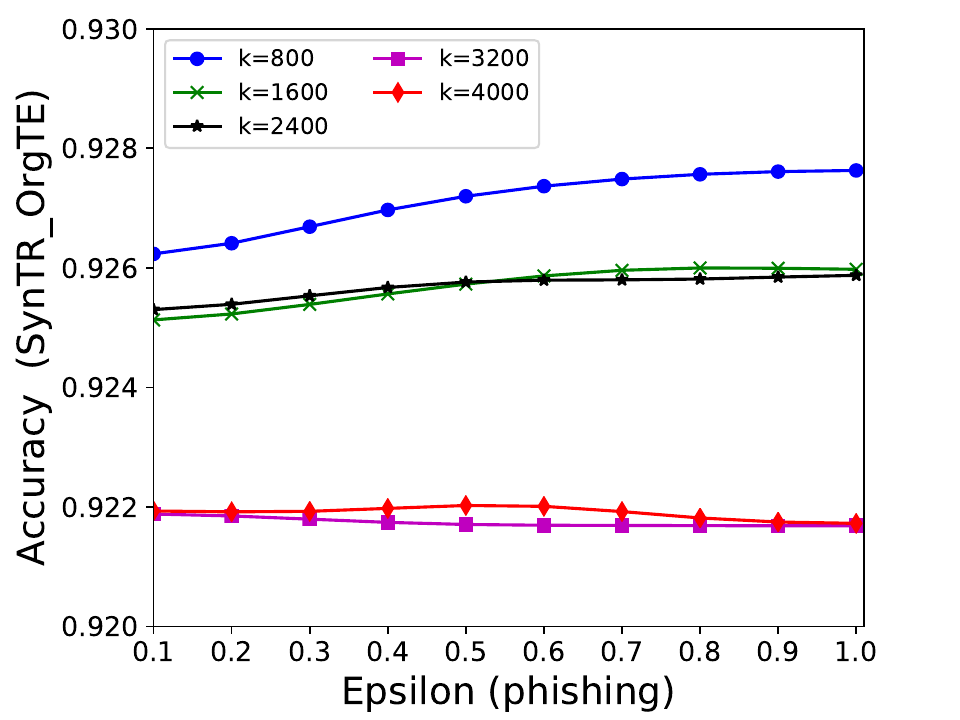}
		\caption{Logistic regression}
        \label{phishing_SynTR_LR_comparek}
	\end{subfigure}
	\begin{subfigure}{.32\textwidth}
		\includegraphics[width=\textwidth]{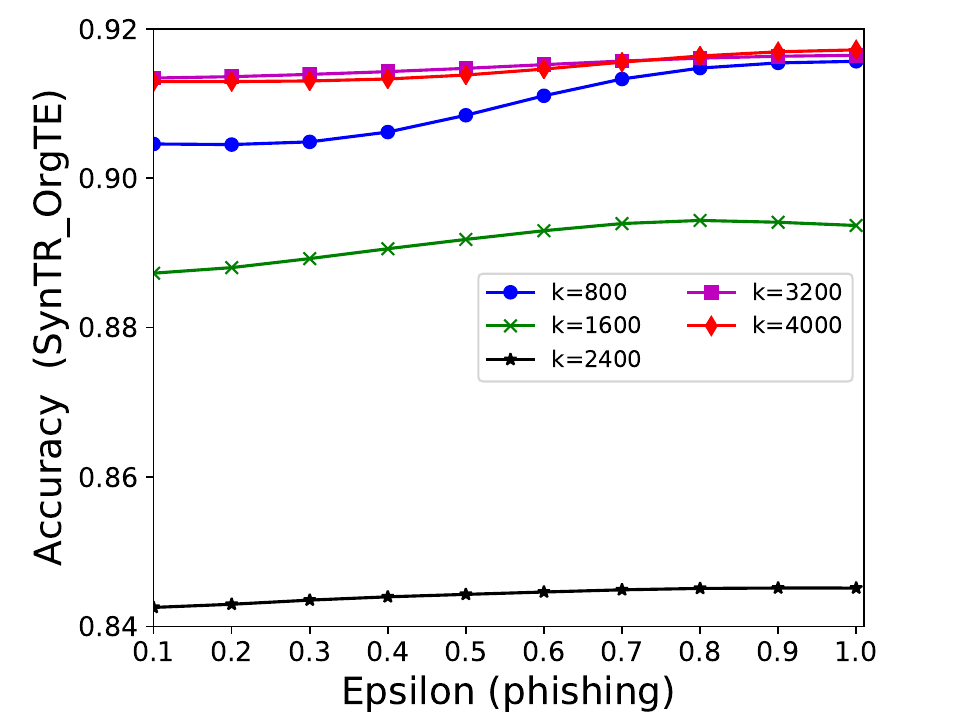}
		\caption{Gradient boosting}
        \label{phishing_SynTR_GB_comparek}
	\end{subfigure}
	\caption{Effect of cluster size (k) on phishing dataset in scenario 2}
    \label{phishing_comparek_synTR}
\end{figure*}
%

\begin{figure*}[!t]
\centering
	\begin{subfigure}{.32\textwidth}
		\includegraphics[width=\textwidth]{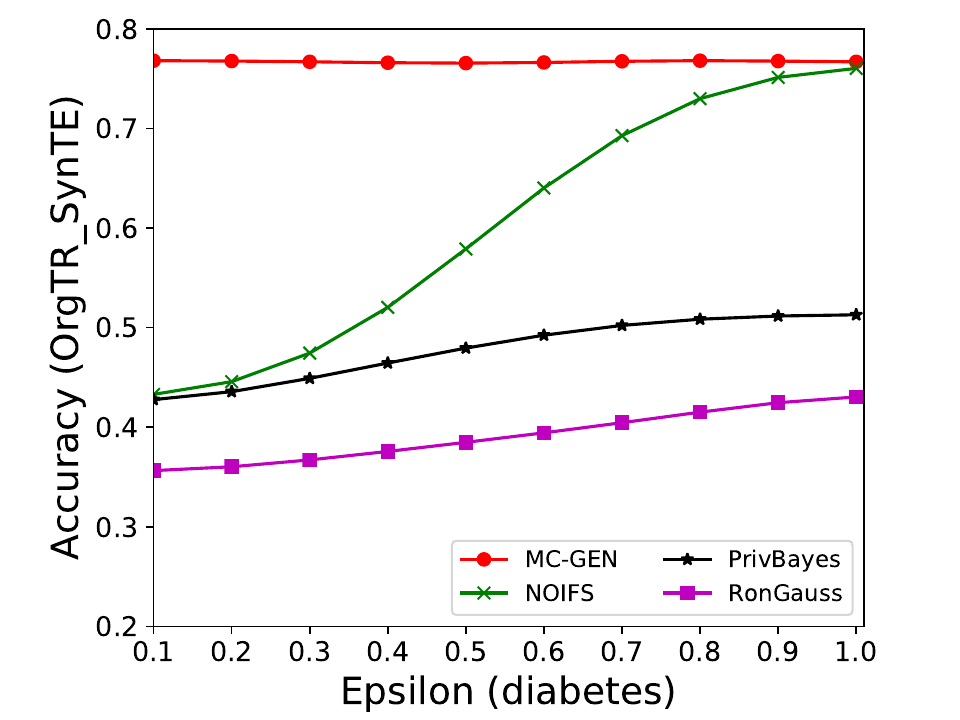}
		\caption{SVM}
        \label{diabetes_OrgTR_SVM_compareMethod}
	\end{subfigure}
	\begin{subfigure}{.32\textwidth}
		\includegraphics[width=\textwidth]{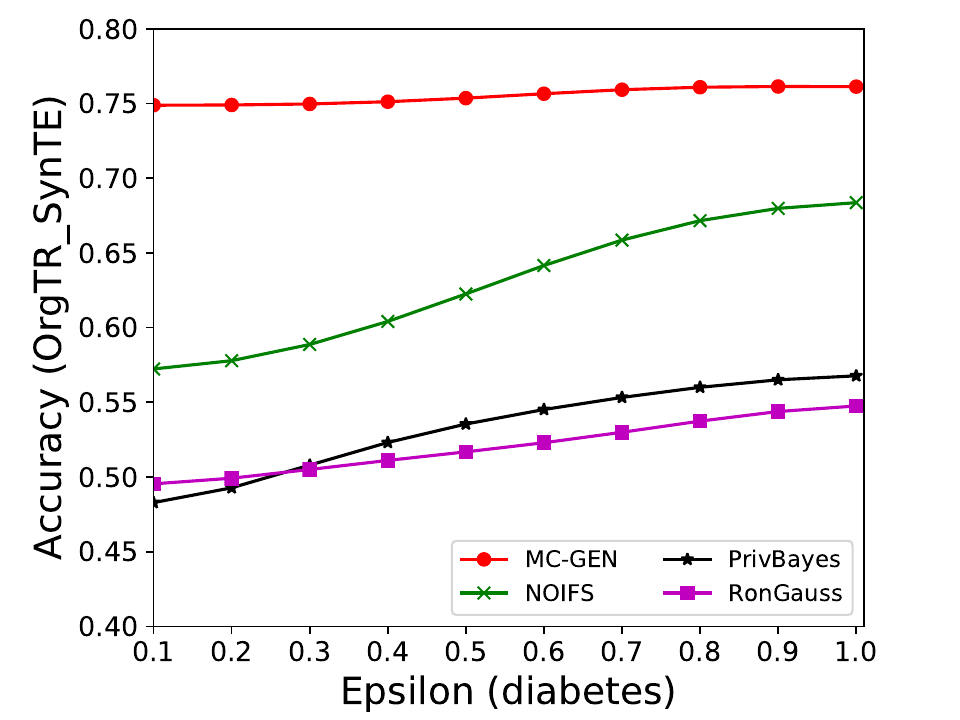}
		\caption{Logistic regression}
        \label{diabetes_OrgTR_LR_compareMethod}
	\end{subfigure}
	\begin{subfigure}{.32\textwidth}
		\includegraphics[width=\textwidth]{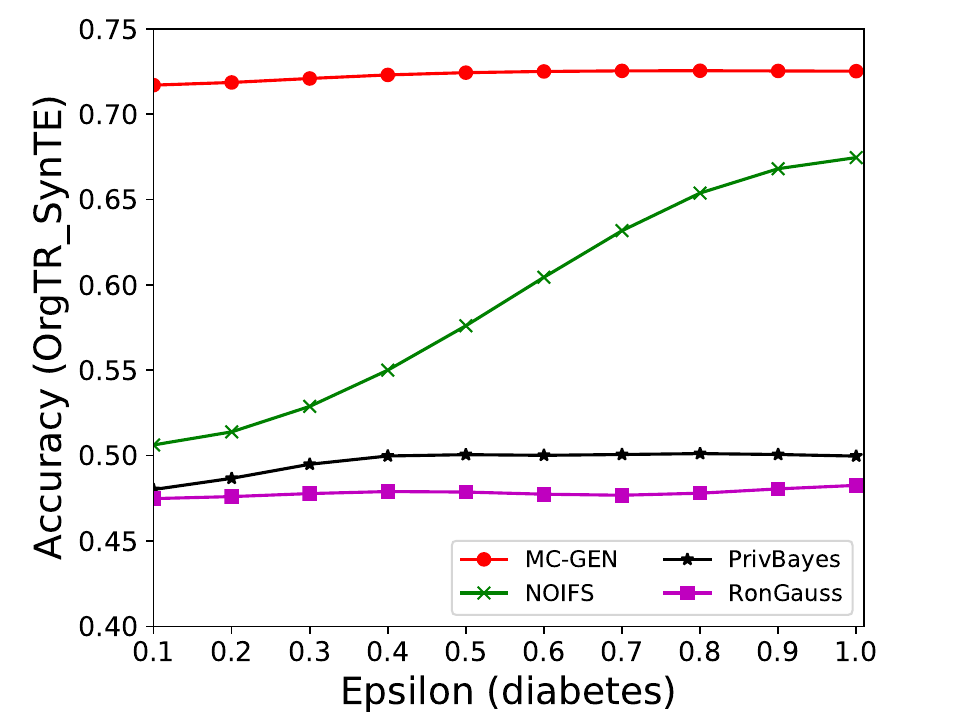}
		\caption{Gradient boosting}
        \label{diabetes_OrgTR_GB_compareMethod}
	\end{subfigure}
	\caption{Comparison with other generation methods on diabetes dataset in scenario 1}
    \label{diabetes_compareMethod_orgTR}
\end{figure*}
%

\begin{figure*}[!t]
\centering
	\begin{subfigure}{.32\textwidth}
		\includegraphics[width=\textwidth]{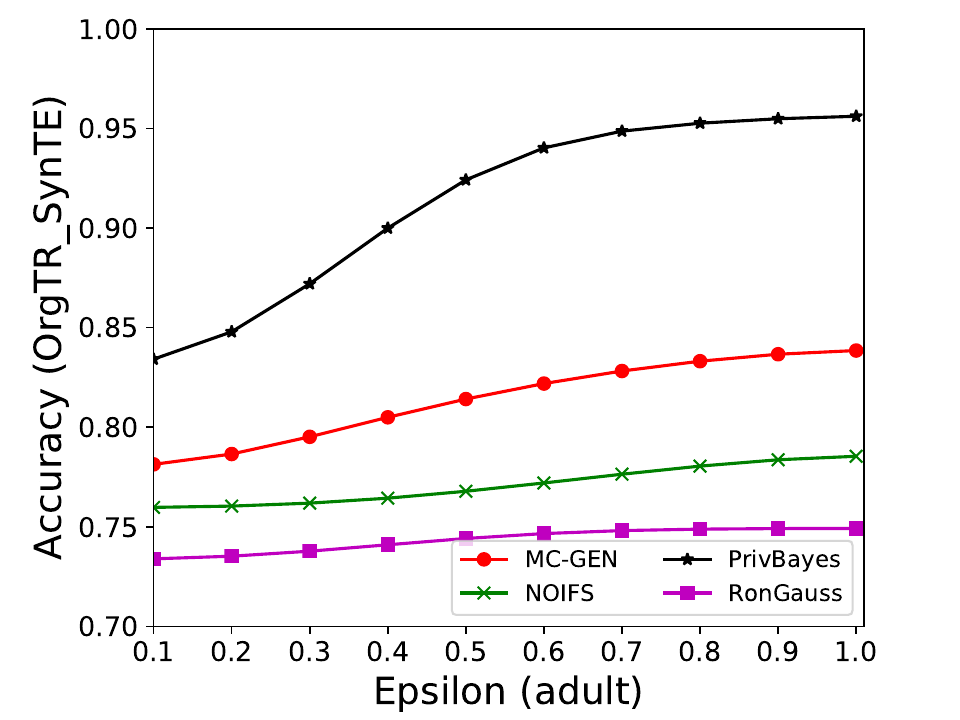}
		\caption{SVM}
        \label{adult9_OrgTR_SVM_compareMethod}
	\end{subfigure}
	\begin{subfigure}{.32\textwidth}
		\includegraphics[width=\textwidth]{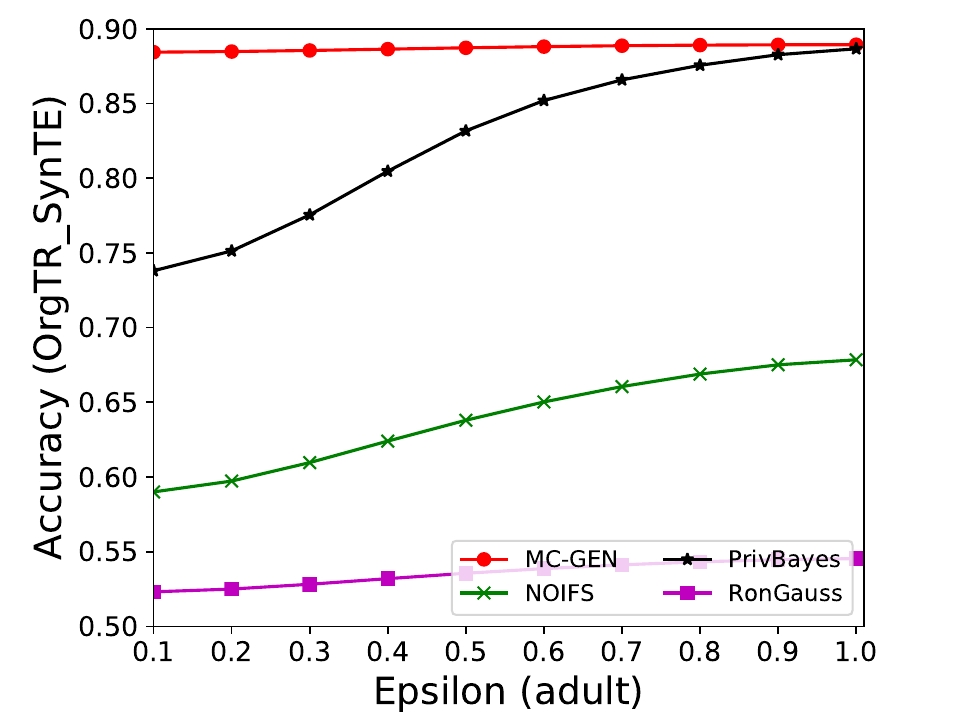}
		\caption{Logistic regression}
        \label{adult9_OrgTR_LR_compareMethod}
	\end{subfigure}
	\begin{subfigure}{.32\textwidth}
		\includegraphics[width=\textwidth]{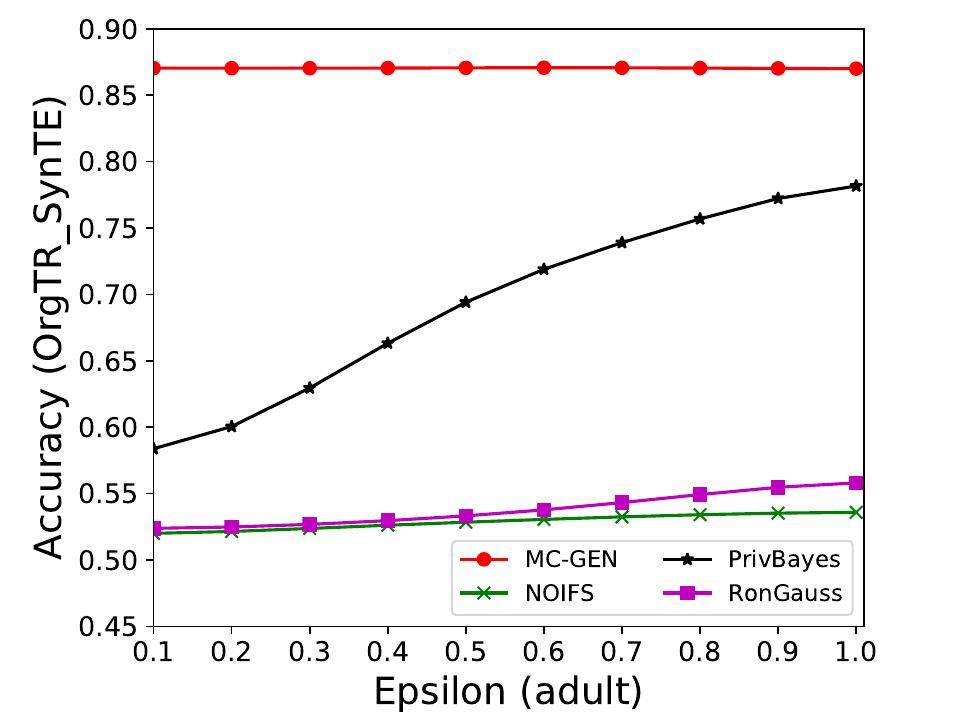}
		\caption{Gradient boosting}
        \label{adult9_OrgTR_GB_compareMethod}
	\end{subfigure}
	\caption{Comparison with other generation methods on adult dataset in scenario 1}
    \label{adults_compareMethod_orgTR}
\end{figure*}
%

\begin{figure*}[!t]
\centering
	\begin{subfigure}{.32\textwidth}
		\includegraphics[width=\textwidth]{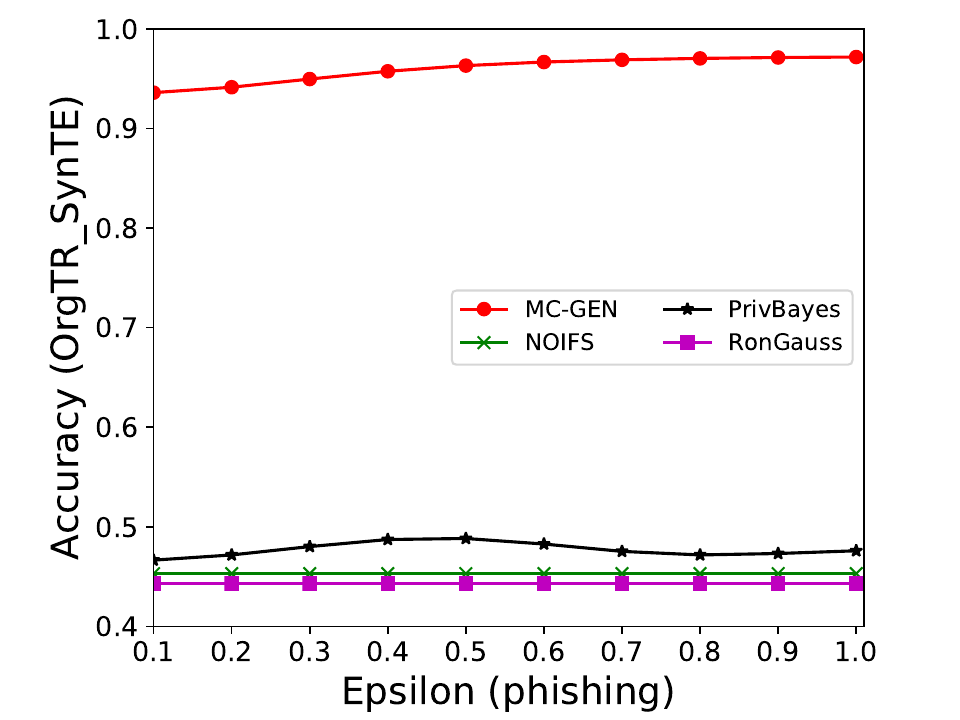}
		\caption{SVM}
        \label{phishing_OrgTR_SVM_compareMethod}
	\end{subfigure}
	\begin{subfigure}{.32\textwidth}
		\includegraphics[width=\textwidth]{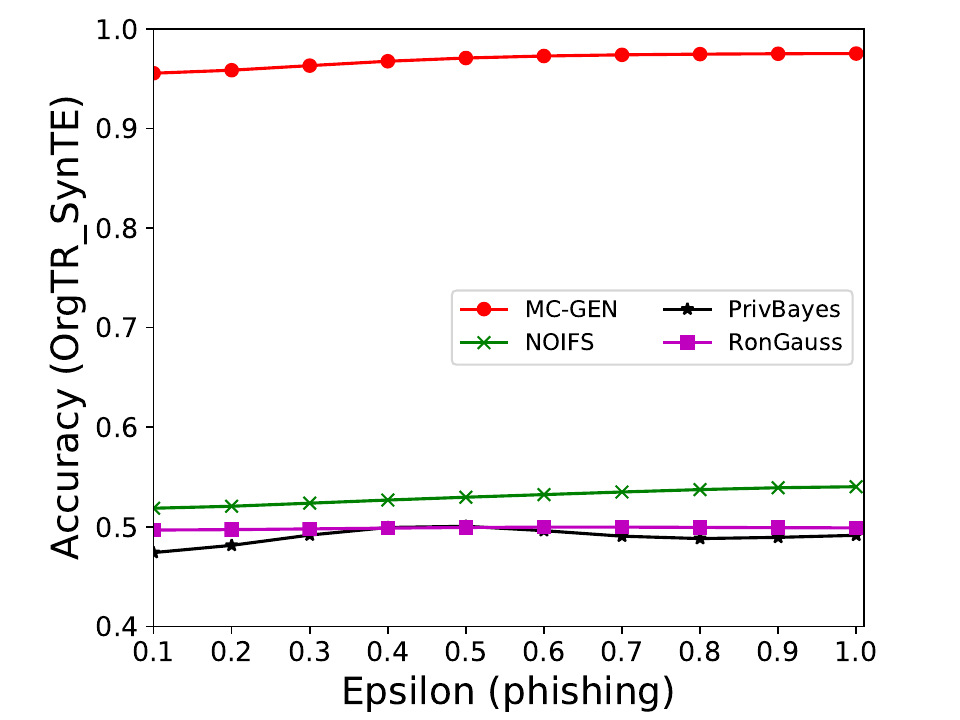}
		\caption{Logistic regression}
        \label{phishing_OrgTR_LR_compareMethod}
	\end{subfigure}
	\begin{subfigure}{.32\textwidth}
		\includegraphics[width=\textwidth]{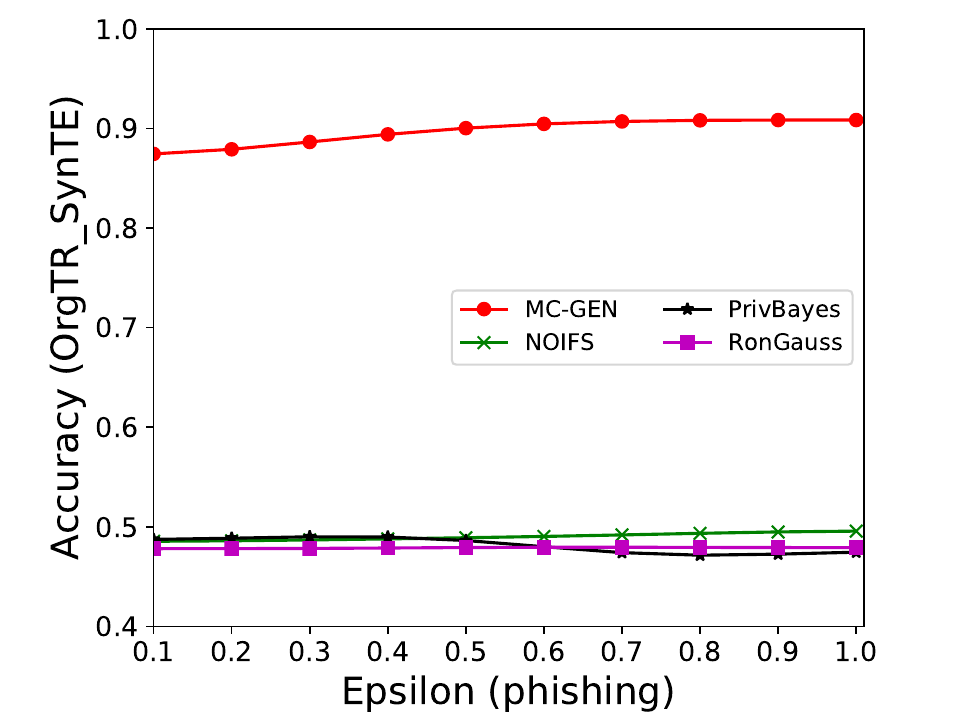}
		\caption{Gradient boosting}
        \label{phishing_OrgTR_GB_compareMethod}
	\end{subfigure}
	\caption{Comparison with other generation methods on phishing dataset in scenario 1}
    \label{phishing_compareMethod_orgTR}
\end{figure*}
%

\begin{figure*}[!t]
\centering
	\begin{subfigure}{.32\textwidth}
		\includegraphics[width=\textwidth]{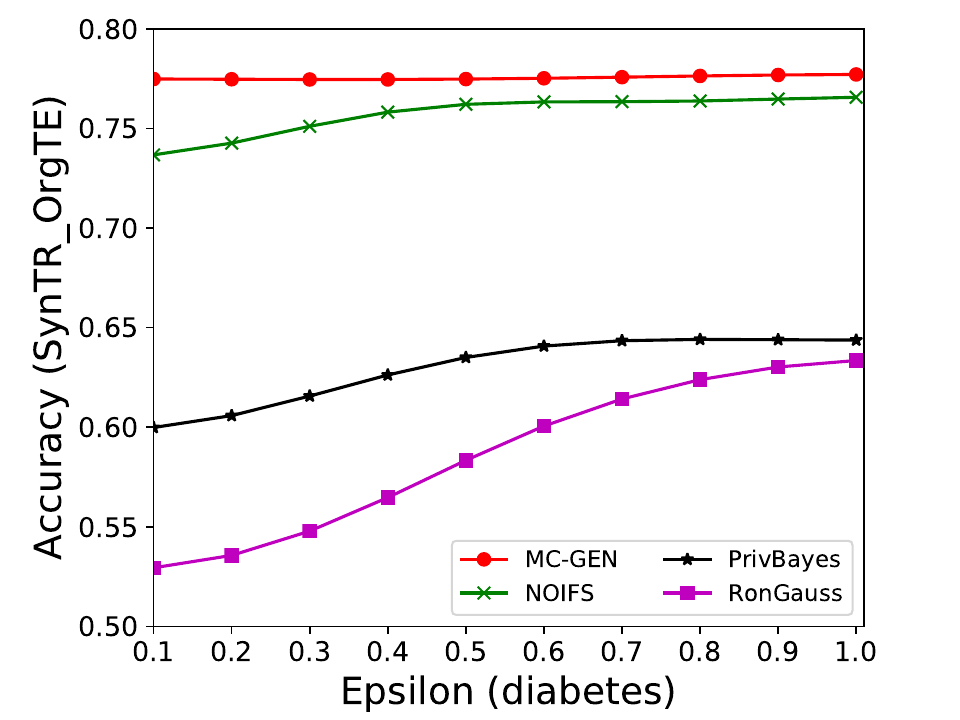}
		\caption{SVM}
        \label{diabetes_SynTR_SVM_compareMethod}
	\end{subfigure}
	\begin{subfigure}{.32\textwidth}
		\includegraphics[width=\textwidth]{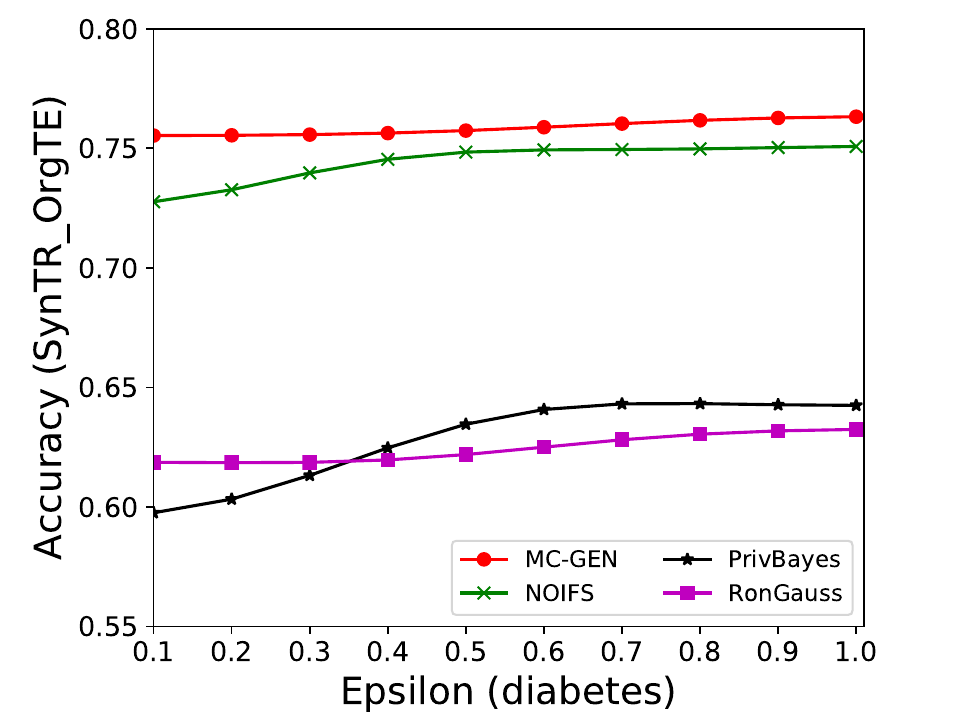}
		\caption{Logistic regression}
        \label{diabetes_SynTR_LR_compareMethod}
	\end{subfigure}
	\begin{subfigure}{.32\textwidth}
		\includegraphics[width=\textwidth]{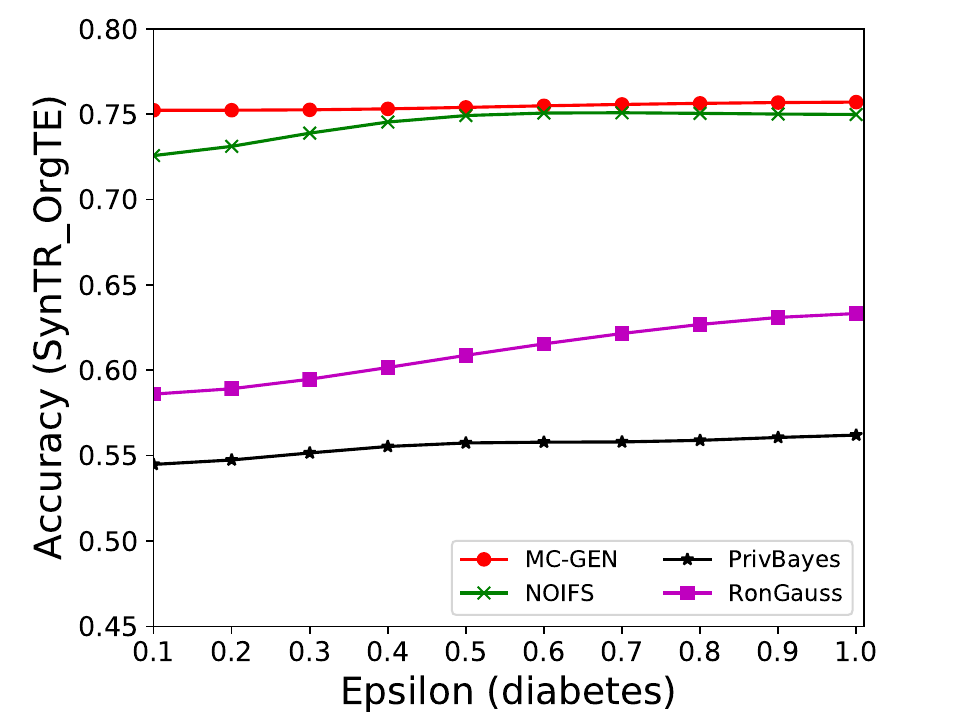}
		\caption{Gradient boosting}
        \label{diabetes_SynTR_GB_compareMethod}
	\end{subfigure}
	\caption{Comparison with other generation methods on diabetes dataset in scenario 2}
    \label{diabetes_compareMethod_synTR}
\end{figure*}
%

\begin{figure*}[!t]
\centering
	\begin{subfigure}{.32\textwidth}
		\includegraphics[width=\textwidth]{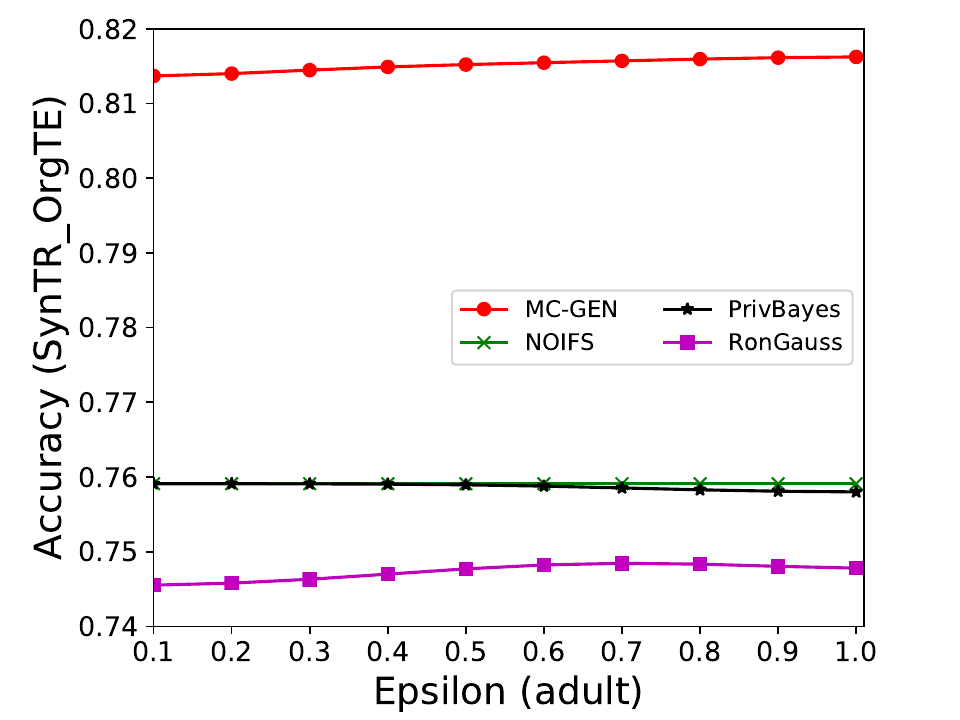}
		\caption{SVM}
        \label{adult9_SynTR_SVM_compareMethod}
	\end{subfigure}
	\begin{subfigure}{.32\textwidth}
		\includegraphics[width=\textwidth]{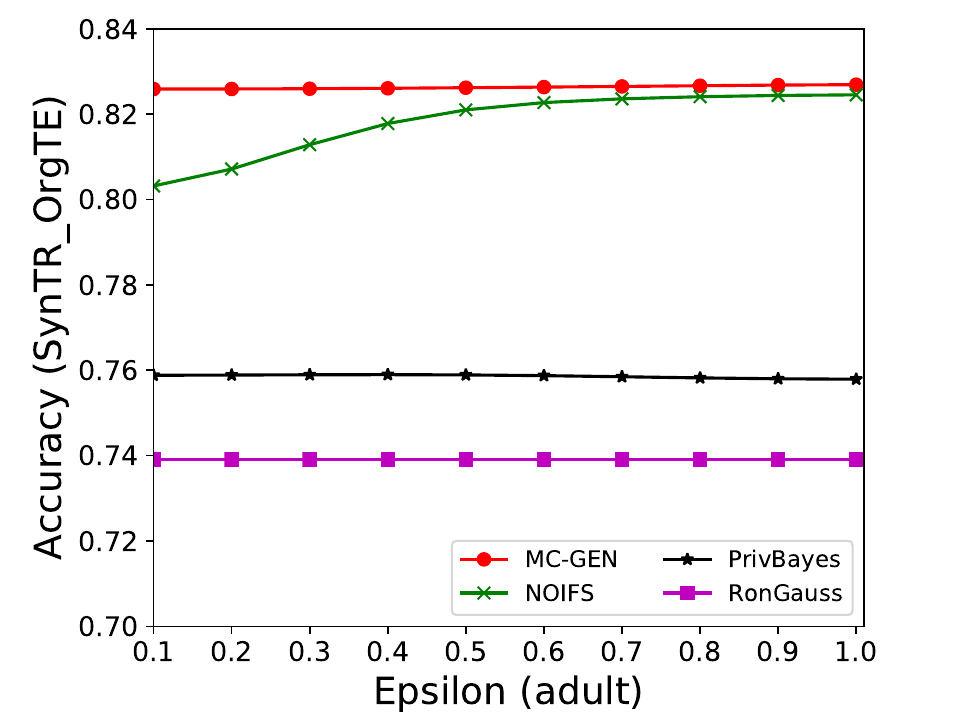}
		\caption{Logistic regression}
        \label{adult9_SynTR_LR_compareMethod}
	\end{subfigure}
	\begin{subfigure}{.32\textwidth}
		\includegraphics[width=\textwidth]{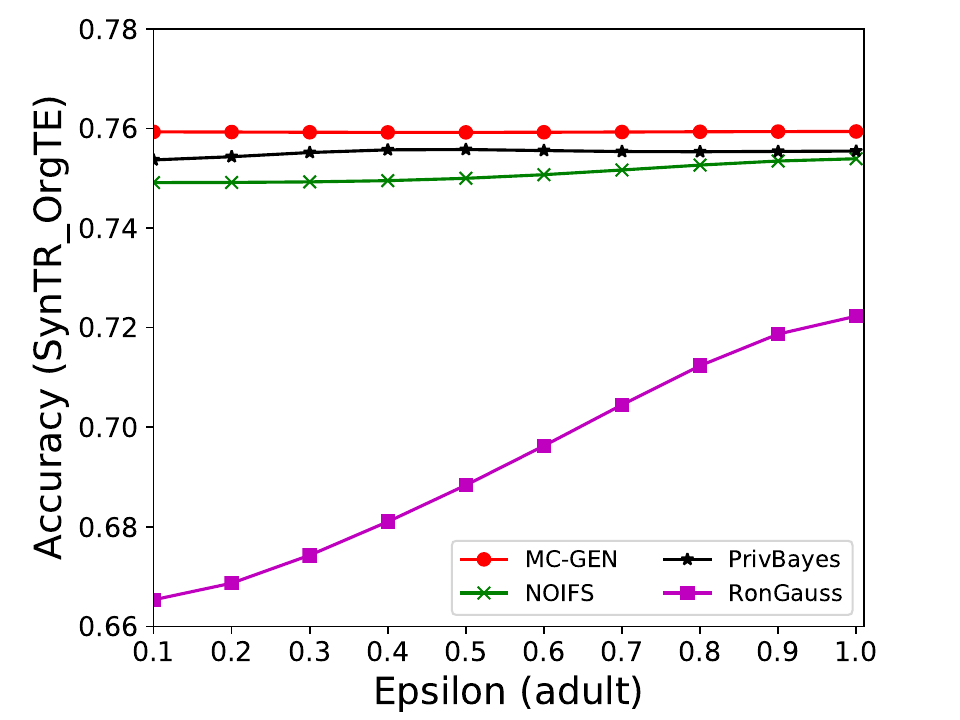}
		\caption{Gradient boosting}
        \label{adult9_SynTR_GB_compareMethod}
	\end{subfigure}
	\caption{Comparison with other generation methods on adult dataset in scenario 2}
    \label{adults_compareMethod_synTR}
\end{figure*}
%

\begin{figure*}[!t]
\centering
	\begin{subfigure}{.32\textwidth}
		\includegraphics[width=\textwidth]{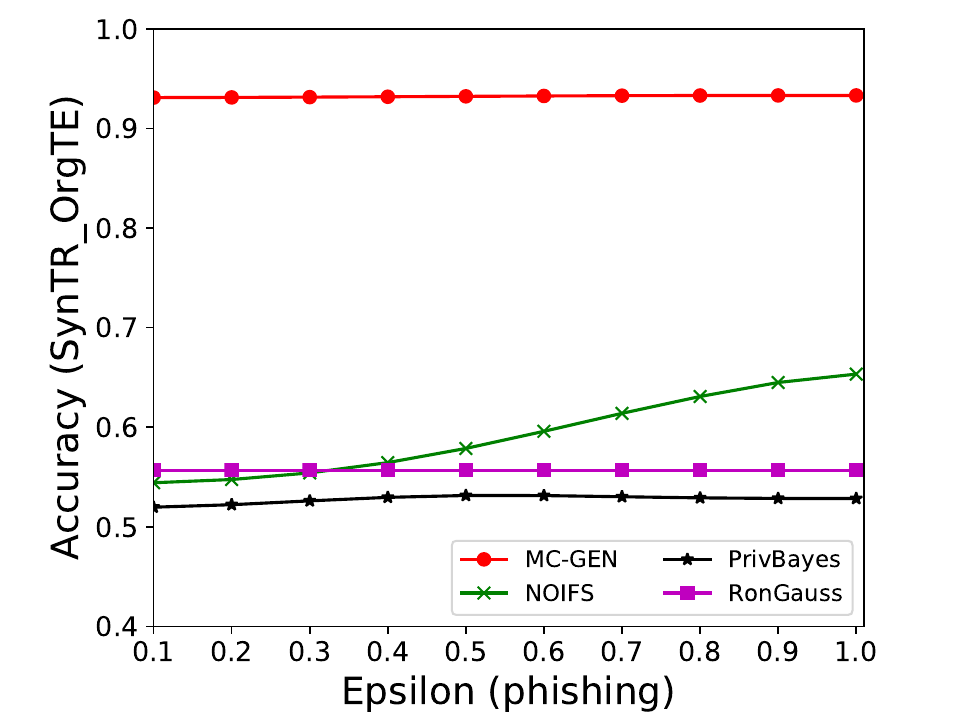}
		\caption{SVM}
        \label{phishing_SynTR_SVM_compareMethod}
	\end{subfigure}
	\begin{subfigure}{.32\textwidth}
		\includegraphics[width=\textwidth]{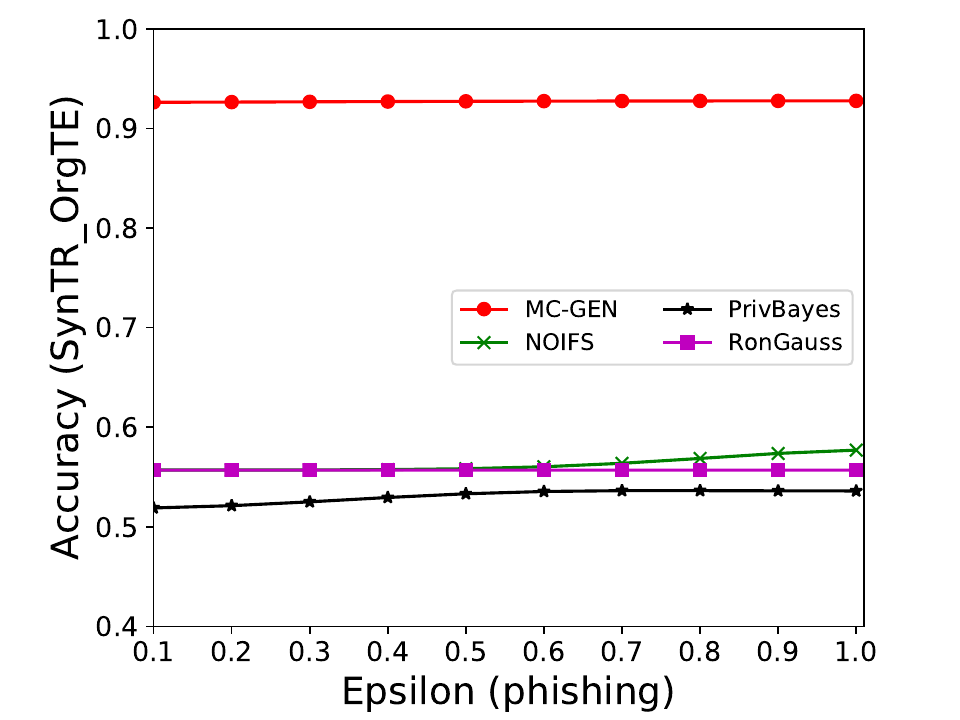}
		\caption{Logistic regression}
        \label{phishing_SynTR_LR_compareMethod}
	\end{subfigure}
	\begin{subfigure}{.32\textwidth}
		\includegraphics[width=\textwidth]{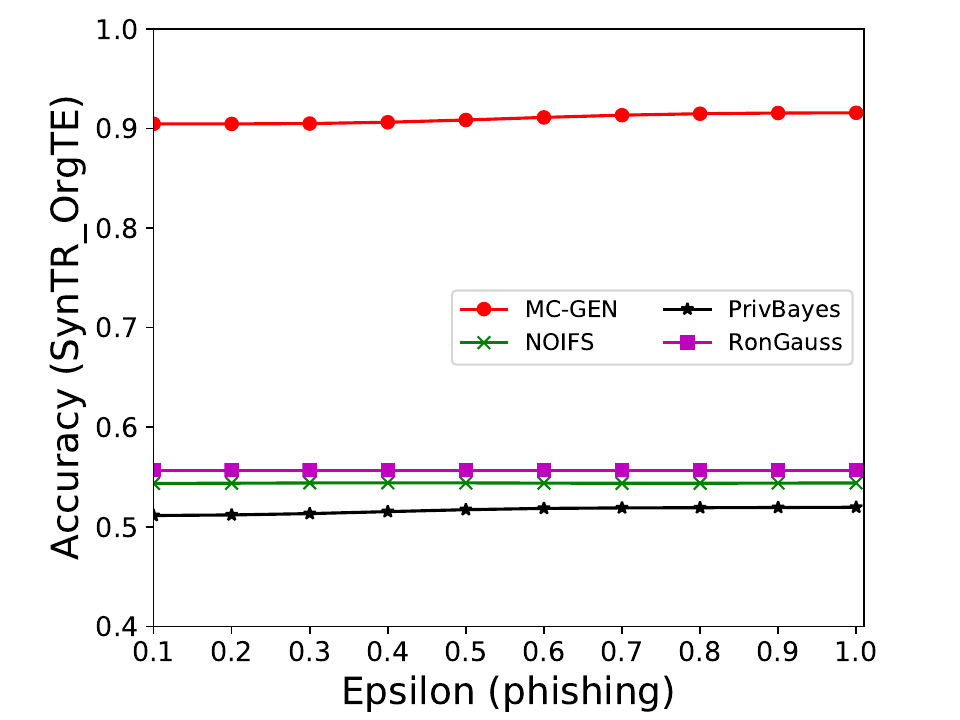}
		\caption{Gradient boosting}
        \label{phishing_SynTR_GB_compareMethod}
	\end{subfigure}
	\caption{Comparison with other generation methods on phishing dataset in scenario 2}
    \label{phishing_compareMethod_synTR}
\end{figure*}

\subsection{Effectiveness Analysis}
In this section, we evaluate the effectiveness of our approach on three classification datasets with three classification algorithms. The best performance of MC-GEN on each dataset is always very close to the baseline (Table. \ref{Baseline of Experiments}). Baselines were trained and tested on the corresponding model using original datasets. Each baseline takes the average performance of 100 experiments.

\subsubsection{Effectiveness Analysis of Different Parameters}
The performance of MC-GEN is affected by two parameters: privacy budget ($\epsilon$) and cluster size ($k$). To evaluate the effectiveness of a single parameter, the other parameter remains the same. As shown in Fig.\ref{diabetes_comparek_orgTR}, \ref{adults_comparek_orgTR}, \ref{phishing_comparek_orgTR}, \ref{diabetes_comparek_synTR}, \ref{adults_comparek_synTR} and \ref{phishing_comparek_synTR}, we observe that:

\begin{enumerate}
\item Privacy budget ($\epsilon$): This parameter controls the noise variance on the synthetic datasets. In DP mechanism, greater $\epsilon$ results in smaller-scaled noise, and vice versa. In our experiments, we investigate $\epsilon$ from 0.1 to 1. As shown in the results, the performance in all classification tasks increased as the privacy budget $\epsilon$ increased. While DP sanitizer adds noise on extracted statistical information (Fig. \ref{syngenerator}), large $\epsilon$ (small-scaled noise) make $\mu\_DP$ and $\Sigma\_DP$ closer to $\mu$ and $\Sigma$, which results in less noise on synthetic datasets. In other words, the generative model captures more accurate statistical information. Thus, a greater $\epsilon$ would result in a better performance.

\item Cluster size ($k$): This parameter controls the number of data points in each cluster. Namely, the number of data points used to build each generative model. We investigate it from 20\% to 100\% based on the total number of seed data. Intuitively, the smaller $k$ would make synthetic data generation more precise since the data is captured in a high-resolution way. However, based on the results shown in Fig. \ref{adults_comparek_orgTR}, \ref{phishing_comparek_orgTR}, \ref{diabetes_comparek_synTR}, \ref{adults_comparek_synTR} and \ref{phishing_comparek_synTR}, small $k$ does not always result in the best performance. For instance, Fig. \ref{adults_comparek_orgTR} (b), Fig. \ref{diabetes_comparek_synTR} (b) and \ref{phishing_comparek_synTR} (a), (b) show that a moderate k achieve the best performance. Fig. \ref{phishing_comparek_orgTR} and Fig. \ref{phishing_comparek_synTR} show that using a large k achieve the best performance. The local optimal k varies on different classification tasks and datasets. As such, it is hard to determine a perfect $k$ that covers all use cases. Finding a local optimal k helps MC-GEN capture the statistical representation of the dataset (clusters) in a good manner. Hence, when the classification task and datasets are assigned, it is worth investigating the local optimal k value for the dataset to ensure performance.
\end{enumerate}

\begin{table*}[t!]
\begin{center}
\begin{tabularx}{0.9\textwidth} { |>{\centering\arraybackslash}X || >{\centering\arraybackslash}X 
  | }
 \hline
 Component& Time Complexity \\
 \hline
 Feature level Clustering & $O(d^{2}log_{d})$\\
 \hline
 Sample level Clustering & $O(n^2)$ \\
 \hline
 Privacy Sanitizer and Generative Model & $O(mn)$\\
 \hline
\end{tabularx}
\caption{Time complexity of MC-GEN}
\label{Time complexity}
\end{center}
\end{table*}

\subsubsection{Comparing with Existing Methods}
In this section, we compared MC-GEN with three existing private synthetic data release methods:  
\begin{itemize}
    \item PrivBayes~\cite{10.1145/2588555.2588573}~\cite{zhang2017privbayes}: PrivBayes is designed on the Bayesian network with differential privacy. Starting from a randomly selected feature node, it extends the network iteratively by selecting a new feature node from the parent set using the exponential mechanism. Then it applied the Laplace mechanism on the conditional probability to achieve the private Bayesian network. The synthetic datasets are generated by using the perturbated Bayesian network. PrivBayes~\cite{10.1145/2588555.2588573}~\cite{zhang2017privbayes} are evaluated on the code provided by the paper authors.
    \item RonGauss~\cite{chanyaswad2017coupling}: RonGauss~\cite{chanyaswad2017coupling} releases the synthetic data based on Gaussian mixture model. It builds the Gaussian mixture model on the projected data in a lower dimension and generates the synthetic dataset based on the Gaussian mixture model with differential privacy noise. RonGauss~\cite{chanyaswad2017coupling} are evaluated on the code provided by the paper authors.
    \item NoIFS: NOIFS follows the same procedure (Fig. \ref{flowchart}) as MC-GEN to generate the synthetic data. The difference between NOIFS and MC-GEN is that NOIFS only applies the sample level clustering (MDAV) in the data prepossessing. Comparing it with MC-GEN helps us to evaluate the effectiveness of feature level clustering.
\end{itemize}

The performance of MC-GEN is evaluated under the local optimal parameter combination. For other methods, we follow the exact settings mentioned in the corresponding paper. Each method is evaluated on an average of 100 experiments. Fig. \ref{diabetes_compareMethod_orgTR}, \ref{adults_compareMethod_orgTR}, \ref{phishing_compareMethod_orgTR}, \ref{diabetes_compareMethod_synTR}, \ref{adults_compareMethod_synTR} and \ref{phishing_compareMethod_synTR} show the results of different methods under the same privacy budgets. We can observe that MC-GEN outperforms the other methods in most cases. Besides, there are several findings during the comparison:

\begin{itemize}
    \item MC-GEN always outperforms NOIFS due to the reduction in noise variance by feature level clustering (proof in section \ref{Statistic Extraction and Differentially Private Sanitizer}). 
    \item  In Fig. \ref{adults_compareMethod_orgTR}, PrivBayes has similar performance as MC-GEN, even outperform on SVM. However, in Fig. \ref{adults_compareMethod_synTR}, MC-GEN outperforms PrivBayes consistently. It may cause by the number of seed data and data types in the datasets. The seed data in scenario 1 is much less than in scenario 2. Limited seed data may affect the performance of MC-GEN. On the other hand, most of the features in adult datasets are categorical features. The Bayesian network can be more effective than the multivariate Gaussian model on categorical data. Therefore, the utility of MC-GEN can be diminished on the small categorical dataset.  
\end{itemize}

\subsection{Time complexity of MC-GEN}
The time complexity of MC-GEN is discussed in this section. In MC-GEN, there are 4 major components: feature level clustering, sample level clustering, privacy sanitizer, and generative model. Assuming the original data set has $n$ samples and $d$ features, feature level clustering vertically divided the dataset into $m$ IFSs. Then, each IFS with corresponding data is horizontally clustered into $j$ clusters. At this point, the original dataset is separated into $m \times j$ clusters and each cluster has $k$ samples. In the end, the privacy sanitizer and generative model are applied simultaneously on each cluster. Since privacy sanitizer and generative model take effect at the same time, we consider it as the same part while discussing the time complexity. Table~\ref{Time complexity} illustrates the detailed time complexity of each component. Feature level clustering mainly relies on agglomerative hierarchical clustering which has time complexity $O(d^{2}log_{d})$. The time complexity of sample level clustering can be found in~\cite{oganian2001complexity}. For privacy sanitizer and generative model, it applies the noise on data in each cluster, so the time complexity should be $O(mjk)$. Since $n=j \times k$, the time complexity can be rewritten as $O(mn)$. All the components are sequentially applied, the total time complexity of MC-GEN is $O(d^{2}log_{d} + n^2 + mn)$.

\section{Conclusion}
\label{sec:Conclusion}
We proposed a novel and effective synthetic data generation method, MC-GEN, which targets on generating a private synthetic dataset for data sharing. We demonstrate MC-GEN improves the utility of synthetic datasets by using multi-level clustering. In the experimental evaluation, we show the effectiveness of synthetic datasets generated by MC-GEN and investigate the parameter effect of MC-GEN. With the best parameter settings, synthetic datasets generated by MC-GEN can achieve similar performance as original datasets. Moreover, we compared MC-GEN with three existing methods. The experimental results show that MC-GEN outperforms the other existing methods in terms of utility. In the future, we will apply and enhance the proposed method on different data types and different machine learning tasks. Meanwhile, as the complexity of data increases, a more powerful and precise model is worth exploring.

\bibliography{elsarticle-template.bbl}

\begin{thebibliography}{10}
\expandafter\ifx\csname url\endcsname\relax
  \def\url#1{\texttt{#1}}\fi
\expandafter\ifx\csname urlprefix\endcsname\relax\def\urlprefix{URL }\fi
\expandafter\ifx\csname href\endcsname\relax
  \def\href#1#2{#2} \def\path#1{#1}\fi

\bibitem{talavera2022data}
E.~Talavera, G.~Iglesias, {\'A}.~Gonz{\'a}lez-Prieto, A.~Mozo,
  S.~G{\'o}mez-Canaval, Data augmentation techniques in time series domain: A
  survey and taxonomy, arXiv preprint arXiv:2206.13508 (2022).

\bibitem{dwork2008differential}
C.~Dwork, Differential privacy: A survey of results, in: International
  Conference on Theory and Applications of Models of Computation, Springer,
  2008, pp. 1--19.

\bibitem{hardt2012simple}
M.~Hardt, K.~Ligett, F.~McSherry, A simple and practical algorithm for
  differentially private data release, in: Advances in Neural Information
  Processing Systems, 2012, pp. 2339--2347.

\bibitem{dwork2011differential}
C.~Dwork, Differential privacy, Encyclopedia of Cryptography and Security
  (2011) 338--340.

\bibitem{taneja2008diffgen}
K.~Taneja, T.~Xie, Diffgen: Automated regression unit-test generation, in:
  Proceedings of the 2008 23rd IEEE/ACM International Conference on Automated
  Software Engineering, IEEE Computer Society, 2008, pp. 407--410.

\bibitem{su2018privpfc}
D.~Su, J.~Cao, N.~Li, M.~Lyu, Privpfc: differentially private data publication
  for classification, The VLDB JournalThe International Journal on Very Large
  Data Bases 27~(2) (2018) 201--223.

\bibitem{bindschaedler2017plausible}
V.~Bindschaedler, R.~Shokri, C.~A. Gunter, Plausible deniability for
  privacy-preserving data synthesis, Proceedings of the VLDB Endowment 10~(5)
  (2017) 481--492.

\bibitem{zhang2017privbayes}
J.~Zhang, G.~Cormode, C.~M. Procopiuc, D.~Srivastava, X.~Xiao, Privbayes:
  Private data release via bayesian networks, ACM Transactions on Database
  Systems (TODS) 42~(4) (2017) 25.

\bibitem{chanyaswad2017coupling}
T.~Chanyaswad, C.~Liu, P.~Mittal, Coupling dimensionality reduction with
  generative model for non-interactive private data release, arXiv preprint
  arXiv:1709.00054 (2017).

\bibitem{soria2017differentially}
J.~Soria-Comas, J.~Domingo-Ferrer, Differentially private data sets based on
  microaggregation and record perturbation, in: Modeling Decisions for
  Artificial Intelligence, Springer, 2017, pp. 119--131.

\bibitem{hall1999correlation}
M.~A. Hall, Correlation-based feature selection for machine learning (1999).

\bibitem{domingo2005ordinal}
J.~Domingo-Ferrer, V.~Torra, Ordinal, continuous and heterogeneous k-anonymity
  through microaggregation, Data Mining and Knowledge Discovery 11~(2) (2005)
  195--212.

\bibitem{sweeney2002k}
L.~Sweeney, k-anonymity: A model for protecting privacy, International Journal
  of Uncertainty, Fuzziness and Knowledge-Based Systems 10~(05) (2002)
  557--570.

\bibitem{https://doi.org/10.48550/arxiv.1802.06739}
L.~Xie, K.~Lin, S.~Wang, F.~Wang, J.~Zhou,
  \href{https://arxiv.org/abs/1802.06739}{Differentially private generative
  adversarial network} (2018).
\newblock \href {https://doi.org/10.48550/ARXIV.1802.06739}
  {\path{doi:10.48550/ARXIV.1802.06739}}.
\newline\urlprefix\url{https://arxiv.org/abs/1802.06739}

\bibitem{9522203}
S.~Imtiaz, M.~Arsalan, V.~Vlassov, R.~Sadre, Synthetic and private smart health
  care data generation using gans, in: 2021 International Conference on
  Computer Communications and Networks (ICCCN), 2021, pp. 1--7.
\newblock \href {https://doi.org/10.1109/ICCCN52240.2021.9522203}
  {\path{doi:10.1109/ICCCN52240.2021.9522203}}.

\bibitem{yoon2018pategan}
J.~Yoon, J.~Jordon, M.~van~der Schaar,
  \href{https://openreview.net/forum?id=S1zk9iRqF7}{{PATE}-{GAN}: Generating
  synthetic data with differential privacy guarantees}, in: International
  Conference on Learning Representations, 2019.
\newline\urlprefix\url{https://openreview.net/forum?id=S1zk9iRqF7}

\bibitem{goodfellow2014generative}
I.~Goodfellow, J.~Pouget-Abadie, M.~Mirza, B.~Xu, D.~Warde-Farley, S.~Ozair,
  A.~Courville, Y.~Bengio, Generative adversarial nets, Advances in neural
  information processing systems 27 (2014).

\bibitem{abs-2112-09238}
Y.~Tao, R.~McKenna, M.~Hay, A.~Machanavajjhala, G.~Miklau,
  \href{https://arxiv.org/abs/2112.09238}{Benchmarking differentially private
  synthetic data generation algorithms}, CoRR abs/2112.09238 (2021).
\newblock \href {http://arxiv.org/abs/2112.09238} {\path{arXiv:2112.09238}}.
\newline\urlprefix\url{https://arxiv.org/abs/2112.09238}

\bibitem{dwork2006calibrating}
C.~Dwork, F.~McSherry, K.~Nissim, A.~Smith, Calibrating noise to sensitivity in
  private data analysis, in: Theory of cryptography conference, Springer, 2006,
  pp. 265--284.

\bibitem{johnson1967hierarchical}
S.~C. Johnson, Hierarchical clustering schemes, Psychometrika 32~(3) (1967)
  241--254.

\bibitem{ahrendt2005multivariate}
P.~Ahrendt, The multivariate gaussian probability distribution, Technical
  University of Denmark, Tech. Rep (2005) 203.

\bibitem{davies1979cluster}
D.~L. Davies, D.~W. Bouldin, A cluster separation measure, IEEE transactions on
  pattern analysis and machine intelligence~(2) (1979) 224--227.

\bibitem{nr}
R.~A. Rossi, N.~K. Ahmed, \href{https://networkrepository.com}{The network data
  repository with interactive graph analytics and visualization}, in: AAAI,
  2015.
\newline\urlprefix\url{https://networkrepository.com}

\bibitem{misc_adult_2}
{Adult}, UCI Machine Learning Repository (1996).

\bibitem{mohammad2012assessment}
R.~M. Mohammad, F.~Thabtah, L.~McCluskey, An assessment of features related to
  phishing websites using an automated technique, in: 2012 international
  conference for internet technology and secured transactions, IEEE, 2012, pp.
  492--497.

\bibitem{scikit-learn}
F.~Pedregosa, G.~Varoquaux, A.~Gramfort, V.~Michel, B.~Thirion, O.~Grisel,
  M.~Blondel, P.~Prettenhofer, R.~Weiss, V.~Dubourg, J.~Vanderplas, A.~Passos,
  D.~Cournapeau, M.~Brucher, M.~Perrot, E.~Duchesnay, Scikit-learn: Machine
  learning in {P}ython, Journal of Machine Learning Research 12 (2011)
  2825--2830.

\bibitem{sklearn_api}
L.~Buitinck, G.~Louppe, M.~Blondel, F.~Pedregosa, A.~Mueller, O.~Grisel,
  V.~Niculae, P.~Prettenhofer, A.~Gramfort, J.~Grobler, R.~Layton,
  J.~VanderPlas, A.~Joly, B.~Holt, G.~Varoquaux, {API} design for machine
  learning software: experiences from the scikit-learn project, in: ECML PKDD
  Workshop: Languages for Data Mining and Machine Learning, 2013, pp. 108--122.

\bibitem{Dua:2019}
D.~Dua, C.~Graff, \href{http://archive.ics.uci.edu/ml}{{UCI} machine learning
  repository} (2017).
\newline\urlprefix\url{http://archive.ics.uci.edu/ml}

\bibitem{chang2011libsvm}
C.-C. Chang, C.-J. Lin, Libsvm: a library for support vector machines, ACM
  transactions on intelligent systems and technology (TIST) 2~(3) (2011) 27.

\bibitem{10.1145/2588555.2588573}
J.~Zhang, G.~Cormode, C.~M. Procopiuc, D.~Srivastava, X.~Xiao,
  \href{https://doi.org/10.1145/2588555.2588573}{Privbayes: Private data
  release via bayesian networks}, SIGMOD '14, Association for Computing
  Machinery, New York, NY, USA, 2014.
\newblock \href {https://doi.org/10.1145/2588555.2588573}
  {\path{doi:10.1145/2588555.2588573}}.
\newline\urlprefix\url{https://doi.org/10.1145/2588555.2588573}

\bibitem{oganian2001complexity}
A.~Oganian, J.~Domingo-Ferrer, On the complexity of optimal microaggregation
  for statistical disclosure control, Statistical Journal of the United Nations
  Economic Commission for Europe 18~(4) (2001) 345--353.

\end{thebibliography}

\end{document}